\documentclass{article}

\usepackage[nonatbib, final]{nips_2016}

\usepackage[utf8]{inputenc} % allow utf-8 input
\usepackage[T1]{fontenc}    % use 8-bit T1 fonts
\usepackage{hyperref}       % hyperlinks
\usepackage{url}            % simple URL typesetting
\usepackage{booktabs}       % professional-quality tables
\usepackage{amsfonts}       % blackboard math symbols
\usepackage{nicefrac}       % compact symbols for 1/2, etc.
\usepackage{microtype}      % microtypography

\usepackage{graphicx}
\usepackage{amsmath}
\usepackage{amssymb}
\usepackage{array}
\usepackage{subfigure}
\usepackage{amsthm}
\usepackage{color}
\usepackage{mathtools}
\usepackage{tikz}
\usepackage{pgfplots}

\newcommand{\defeq}{\vcentcolon=}
\newcommand{\eqdef}{=\vcentcolon}
\newcommand{\norm}[1]{\left\|{#1}\right\|}
\newcommand{\rank}{\operatorname{rank}}
\newcommand{\vect}{\operatorname{vec}}
\newcommand{\var}{\operatorname{Var}}
\newcommand{\ind}[1]{\mathbf{1}\left(#1\right)}
\newcommand{\calB}{\mathcal{B}}
\newcommand{\rmd}{\mathrm{d}}
\newcommand{\rmF}{\mathrm{F}}
\newcommand{\bbR}{\mathbb{R}}
\newcommand{\bbE}{\mathbb{E}}
\newcommand{\bbP}{\mathbb{P}}
\newcommand{\tU}{\widetilde{U}}
\newcommand{\tV}{\widetilde{V}}
\newcommand{\ba}{\bar{a}}
\newcommand{\bb}{\bar{b}}
\newcommand{\ta}{\tilde{a}}
\newcommand{\tb}{\tilde{b}}
\newcommand{\hX}{\widehat{X}}
\newcommand{\hU}{\widehat{U}}
\newcommand{\hV}{\widehat{V}}
\newcommand{\hSigma}{\widehat{\Sigma}}

\newtheorem{theorem}{Theorem}
\newtheorem{lemma}{Lemma}
\newtheorem{corollary}{Corollary}
\newtheorem{proposition}{Proposition}

\title{Joint Dimensionality Reduction\\for Two Feature Vectors}

\author{
  Yanjun~Li\\%\thanks{Webpage: \url{http://yli145.web.engr.illinois.edu}.} \\
  CSL and Department of ECE\\
  University of Illinois, Urbana-Champaign\\
  Urbana, IL 61801 \\
  \texttt{yli145@illinois.edu} \\
  \And
  Yoram~Bresler\\%\thanks{Webpage: \url{http://www.ifp.illinois.edu/~yoram}.} \\
  CSL and Department of ECE\\
  University of Illinois, Urbana-Champaign\\
  Urbana, IL 61801 \\
  \texttt{ybresler@illinois.edu} \\
}

\begin{document}

\maketitle

\begin{abstract}
Many machine learning problems, especially multi-modal learning problems, have two sets of distinct features (e.g., image and text features in news story classification, or neuroimaging data and neurocognitive data in cognitive science research). This paper addresses the joint dimensionality reduction of two feature vectors in supervised learning problems. In particular, we assume a discriminative model where low-dimensional linear embeddings of the two feature vectors are sufficient statistics for predicting a dependent variable. We show that a simple algorithm involving singular value decomposition can accurately estimate the embeddings provided that certain sample complexities are satisfied, without specifying the nonlinear link function (regressor or classifier). 
%These embeddings improve the efficiency and robustness of subsequent training, and can serve as a pre-training algorithm for neural networks. 
The main results establish sample complexities under multiple settings. Sample complexities for different link functions only differ by constant factors.
\end{abstract}

\section{Introduction}
Dimensionality reduction (also known as low-dimensional embedding) is used in machine learning to select and extract features from high dimensional data. Unsupervised learning techniques aim to embed high-dimensional data into low-dimensional features that most accurately represent the original data. The literature on this topic is vast, from classical methods, such as principal component analysis (PCA) and multidimensional scaling (MDS), to more recent approaches, such as Isomap and locally-linear embedding \cite{Tenenbaum2000,Roweis2000}. On the other hand, supervised learning techniques -- a long line of work including linear discriminant analysis (LDA) and canonical correlation analysis (CCA) -- extract features from one set of variables that are most relevant to another set of variables. A related problem is variable selection (also known as feature selection), which selects a subset of active predictors that are relevant to the task. %These techniques help answer interesting questions in data mining, such as which genes are linked to a certain disease or disorder, or what are the determining factors in a basketball game. 

In many real-world machine learning problems, there exist two sets of features with distinct characteristics. For example, while intuitively widely different, both text and images are critical features in machine learning tasks related to news articles \cite{Ramisa2016}. As another example, cognitive science research heavily relies on both neurocognitive data and neuroimaging data, which again are widely different \cite{Noreika2013}. 
This paper studies joint dimensionality reduction of such feature vectors in supervised learning, where an unknown discriminative model $p(y|a,b)$ has two feature vectors $a$ and $b$. We extract two sets of low-dimensional features that are linear combinations of entries in $a$ and $b$, respectively. The linear mappings from $a$ or $b$ to the corresponding features are called \emph{linear embeddings}, which are essentially captured by two subspaces that we call \emph{dimensionality reduction subspaces}. The two embeddings recovered simultaneously do \emph{not} mix the information from $a$ and $b$, leading to more interpretable features crucial to data mining tasks \cite{Chang2015}.
We use a very simple algorithm that involves singular value decomposition (SVD) to estimate the two low-dimensional linear embeddings from i.i.d. samples of the independent variables $a,b$ and the dependent variable $y$. This algorithm does not require any knowledge of underlying model $p(y|a,b)$. Our main results establish the sample complexities under which the embeddings can be accurately estimated. Assume that the ambient dimension of the original data (i.i.d. samples of $a$ and $b$) are $n_1$ and $n_2$, respectively, and we hope to extract $r$ features from each. Then, the sampling complexities for our dimensionality reduction algorithms are as follows:
\begin{enumerate}
	\item In the simple setting where the embeddings are unstructured, $O(n_1n_2)$ samples are sufficient to estimate the $r$-dimensional embeddings accurately (\emph{Section \ref{sec:dr}}).
	\item If $s_1$ (resp. $s_2$) variables are selected from $n_1$ (resp. $n_2$) variables, and are in turn reduced to $r$ features each, then the required sample complexity is $O(s_1s_2\log n_1 \log n_2)$ (\emph{Section \ref{sec:vs}}).
	\item If the dependent variable $y$ is a light-tailed random variable, $O((n_1+n_2)\log^6(n_1+n_2))$ samples are sufficient (\emph{Section \ref{sec:optimal}}).
\end{enumerate}
These sample complexity results hold under mild conditions. Here, we assume that $r=O(1)$ for simplicity, so that the explicit dependencies of the sample complexities on $r$ are hidden. As an example, we derive such explicit dependencies for the bilinear link function in Section \ref{sec:bivariate}.  %and do not hinge on the complexity of the link function. In fact, the sample complexity bounds for different underlying models only differ by a constant factor.

The estimators in this paper can serve several practical purposes. First, the linear embeddings extract features that best explain the dependent variable, which is of interest to many data mining problems. Secondly, by reducing the number of variables, low-dimensional embeddings challenge the curse of dimensionality and enable faster and more robust training in subsequent stages. Lastly, even if the embedding estimates are error-prone due to lack of a sufficient number of samples, they can be used to initialize more sophisticated training algorithms. For example, in a neural network setting, the embeddings in this paper are estimates of weights in the first layer of the network, which is a method of pre-training \cite{Hinton2006}. Then the weights can be fine-tuned using back propagation.

There has been a long line of research in supervised dimensionality reduction, to name a few examples, sliced inverse regression (SIR) \cite{Li1991}, principal Hessian direction (pHd) \cite{Li1992}, sliced average variance estimation (SAVE) \cite{Cook2000}, and minimum average variance estimation (MAVE) \cite{Xia2002}. However, none of these approaches studies the joint dimensionality reduction of two feature vectors. When the link function is odd in both variables (e.g., a bilinear function), SIR, pHd, and SAVE cannot recover the embeddings. MAVE is based on local linear approximations, hence it is not applicable to non-smooth link functions. Recently, Plan et al. \cite{Plan2014} studied the generalized linear model, which corresponds to extracting one feature from one vector. We extend their analysis to jointly extracting multiple features from two vectors. Our approach is a new member in the family of supervised dimensionality reduction algorithms, which applies to multi-modal learning problems and overcomes the drawbacks of previous approaches in this setting.

%%%%%%%%%%%%%%%%%%%%%%%%%%%%%% Bilinear Regression %%%%%%%%%%%%%%%%%%%%%%%%%%%%%%

\subsection{Linear Estimator for Bilinear Regression}\label{sec:bilinear}
As a warmup, we review an interesting result for bilinear regression.
Suppose random variable $y\in\bbR$ satisfies $y = \left<ab^T,X\right> = a^TXb$,
where random variables $a\in\bbR^{n_1}$ and $b\in\bbR^{n_2}$ are independent, following probability distributions that satisfy isotropy: $\bbE[a_i a_i^T]=I_{n_1}$ and $\bbE[b_i b_i^T]=I_{n_2}$, respectively (e.g., $a_i\sim N(0,I_{n_1})$ and $b_i\sim N(0,I_{n_2})$). The matrix $X\in\bbR^{n_1\times n_2}$ is fixed but unknown. Given $m$ i.i.d. observations $\{y_i\}_{i=1}^m$, $\{a_i\}_{i=1}^m$, and $\{b_i\}_{i=1}^m$, $\hX_\mathrm{lin} \defeq \frac{1}{m} \sum_{i=1}^{m} a_i y_i b_i^T$ is an unbiased linear estimator of $X$:
\[
\bbE[\hX_\mathrm{lin}] =  \frac{1}{m}\sum_{i=1}^{m} \bbE[a_i y_i b_i^T] = \bbE[a_1 y_1 b_1^T] = \bbE[a_1 a_1^T X b_1 b_1^T] =  \bbE[a_1 a_1^T]\cdot X \cdot\bbE[b_1 b_1^T] = X.
\]
In some applications, we have prior knowledge of the matrix $X$ -- it belongs to a subset $\Omega$ of $\bbR^{n_1\times n_2}$, for example, $X$ has at most rank $r$, or has at most $s_1$ nonzero rows and at most $s_2$ nonzero columns. Then one can project the linear estimator onto the subset, obtaining a nonlinear estimator $\hX = P_{\Omega}\hX_\mathrm{lin}$. This estimator is used to initialize algorithms for matrix recovery with rank-1 measurement matrices (e.g., phase retrieval and blind deconvolution via lifting \cite{Netrapalli2013,Lee2015a}).

%%%%%%%%%%%%%%%%%%%%%%%%%%%%%% General Case %%%%%%%%%%%%%%%%%%%%%%%%%%%%%%
\subsection{Learning with Two Feature Vectors}\label{sec:nonlinear}
Suppose random variable $y$ depends on $a$ and $b$ only through $U^T a$ and $V^T b$, i.e. we have the following Markov chain:
\begin{align}
(a,b) \rightarrow (U^T a, V^T b) \rightarrow y, \label{eq:markov}
\end{align}
where $U\in \bbR^{n_1\times r}$ and $V\in \bbR^{n_2\times r}$ are unknown tall matrices.
In machine learning, $p\left(y\middle|a,b\right)=p\left(y\middle| U^Ta,V^Tb\right)$ corresponds to the discriminative model. In communication, it corresponds to a multiple-inputs-single-output (MISO) channel with inputs $U^Ta,V^Tb$ and output $y$. Clearly, there exists a deterministic bivariate functional $f(\cdot,\cdot)$ such that $\bbE[y|a,b] = \mu = f(U^T a, V^T b)$, and the randomness of $\mu$ comes from $U^T a$ and $V^T b$. Moreover, assume that
\begin{align}
\var[y|a,b] \leq \sigma_{y|a,b}^2, \label{eq:con_var}
\end{align}
where $\sigma_{y|a,b}^2$ is a constant upper bound for the conditional variance. When $y=f(U^T a, V^T b) = \left<U^T a,V^T b\right>$, the above nonlinear regression reduces to the bilinear regression in Section \ref{sec:bilinear}, for which $X = UV^T$.

In a special case, $y$ depends on $a$ and $b$ only through $\mu$ (rather than through $U^Ta$ and $V^T b$), i.e.,
\begin{align}
(a,b) \rightarrow (U^T a, V^T b) \rightarrow \mu = f(U^T a, V^T b) \rightarrow y, \label{eq:markov2}
\end{align}
We give two examples of the conditional distribution $p(y|\mu)$:
\begin{enumerate}
	\item Gaussian distribution. Let $y = \mu + z$, where $z \sim N(0,\sigma_z^2)$.	This corresponds to additive Gaussian noise in the observation, and the tightest bound is $\sigma_{y|a,b}^2 = \sigma_z^2$.
	\item Bernoulli distribution. In binary classification, the conditional mean $\mu$ of the binary label $y$ belongs to the interval $[0,1]$, and
	\[
	y = \begin{cases}
	1 &~ \text{w.p.}~ \mu\\
	0 &~ \text{w.p.}~ 1-\mu
	\end{cases}
	\quad \sim \operatorname{Ber}(\mu).
	\]
	Hence $\sigma_{y|a,b}^2 = \max\limits_{\mu\in[0,1]}\mu(1-\mu) = \frac{1}{4}$. The conditional mean in this model can take many forms, two of which are:
	\begin{itemize}
	\item \emph{Logistic-type function} $\mu = 1/\left(1+e^{-g(U^T a, V^T b)}\right)$.
	\item \emph{Indicator-type function} $\mu = \epsilon+(1-2\epsilon)\cdot\ind{g(U^T a, V^T b)>0}$, where $\ind{\cdot}$ denotes the indicator function, and $\epsilon$ denotes noise in the labels. When $\epsilon = 0$, $\mu$ is either $1$ or $0$, and all samples are correctly labeled. When $\epsilon \in (0,\frac{1}{2})$, $\mu$ is either $1-\epsilon$ or $\epsilon$, and every sample is mislabeled with probability $\epsilon$.
	\end{itemize}
\end{enumerate}
In the rest of the paper, we assume only \eqref{eq:markov} and \eqref{eq:con_var} in our analysis. The sole purpose of the special case \eqref{eq:markov2} is to demonstrate the connections of our model with various machine learning models. %More generally, different moments of $y$, or even $p\left(y\middle|U^Ta,V^Tb\right)$, may dependent on $U^Ta,V^Tb$ through different functions.
Estimation of $U$ and $V$ corresponds to joint dimensionality reduction of two feature vectors, which plays an important role in machine learning with high-dimensional multi-modal data. Once we estimate $U$ and $V$, the number of input random variables are reduced from $n_1+n_2$ to $2r$.

%%%%%%%%%%%%%%%%%%%%%%%%%%%%%% r Dim Reduction %%%%%%%%%%%%%%%%%%%%%%%%%%%%%%
\section{Dimensionality Reduction}\label{sec:dr}
Suppose $a\in\bbR^{n_1}$ and $b\in\bbR^{n_2}$ follow Gaussian distributions $N(0,I_{n_1})$ and $N(0,I_{n_2})$, respectively. 
We establish the following interesting result: given i.i.d. observations $\{y_i\}_{i=1}^m$, $\{a_i\}_{i=1}^m$, and $\{b_i\}_{i=1}^m$, we can estimate the subspaces encoded by $U$ and $V$, even if the nonlinear functional $f(\cdot,\cdot)$ is unspecified or nonparametric. 

Without loss of generality, we assume that $U$ and $V$ have orthonormal columns. Let $\tU\in\bbR^{n_1\times(n_1-r)}$ and $\tV\in\bbR^{n_2\times(n_2-r)}$ be matrices of orthonormal columns that satisfy $U^T\tU=0$, $V^T\tV = 0$, i.e., the columns of $\tU$ and $\tV$ span the orthogonal complements of the subspaces spanned by the columns of $U$ and $V$. Define $\ba_i \defeq U^T a_i$, $\ta_i \defeq \tU^T a_i$, $\bb_i \defeq V^T b_i$, and $\tb_i \defeq \tV^T b_i$.
\begin{lemma}\label{lem:r_ind}
$\{\ba_i\}_{i=1}^m$, $\{\ta_i\}_{i=1}^m$, $\{\bb_i\}_{i=1}^m$, and $\{\tb_i\}_{i=1}^m$ are all independent Gaussian random vectors. Moreover, $\ba_i\sim N(0,I_{r})$, $\ta_i\sim N(0,I_{n_1-r})$, $\bb_i\sim N(0,I_{r})$, $\tb_i\sim N(0,I_{n_2-r})$.
\end{lemma}

From now on, we assume $f(\cdot,\cdot)$ is such that the following quantities are finite:
\begin{align}\label{eq:defQ}
\begin{split}
Q \defeq \bbE\left[\ba_1 f(\ba_1,\bb_1) \bb_1^T \right],& \quad \sigma^2 \defeq \bbE\left[\norm{\ba_1 f(\ba_1,\bb_1) \bb_1^T-Q}_\rmF^2 \right],\\ 
\tau_0^2 \defeq \bbE\left[\left|f(\ba_1,\bb_1)\right|^2\right], \quad \tau_1^2 \defeq \bbE&\left[\norm{\ba_1f(\ba_1,\bb_1)}_2^2\right], \quad \tau_2^2 \defeq \bbE\left[\norm{\bb_1f(\ba_1,\bb_1)}_2^2 \right].
\end{split}
\end{align}
By Lemma \ref{lem:r_ind}, $\ba_1\sim N(0,I_r)$ and $\bb_1\sim N(0,I_r)$ are independent Gaussian random variables. Therefore, $Q,\sigma,\tau_0,\tau_1,\tau_2$ are constants that only depend on $f(\cdot,\cdot)$ and $r$, and not on $n_1$, $n_2$, and $m$. 

Theorem \ref{thm:lin} shows that
\begin{align}
\hX_\mathrm{lin} \defeq \frac{1}{m} \sum_{i=1}^{m} a_i y_i b_i^T \label{eq:lin}
\end{align}
is an unbiased linear estimator of $X = UQV^T$.

\begin{theorem}\label{thm:lin}
The linear estimator $\hX_\mathrm{lin}$ in \eqref{eq:lin} satisfies:
\[
\bbE \left[ \hX_\mathrm{lin} \right] = X \defeq UQV^T,
\]
\[
\bbE \left[ \norm{\hX_\mathrm{lin}-X}_\rmF^2 \right] \leq \frac{1}{m}\left[n_1n_2\sigma_{y|a,b}^2 + \sigma^2+(n_1-r)(n_2-r)\tau_0^2+(n_2-r)\tau_1^2+ (n_1-r)\tau_2^2\right].
\]
\end{theorem}

Let $\hU\hSigma\hV^T$ be the best rank-$r$ approximation of $\hX_\mathrm{lin}$, containing the first $r$ singular values and singular vectors. If $Q$ is nonsingular, then $\hU,\hV$ can be used to estimate $U,V$ up to rotation ambiguity.\footnote{There exist orthogonal matrices $Q_1,Q_2\in\bbR^{r\times r}$ such that $\|\hU-UQ_1\|_\rmF$ and $\|\hV-VQ_2\|_\rmF$ are bounded. Rotation ambiguity does not pose any problems, since the subspaces encoded by $U,V$ are invariant to rotations.}
We denote the smallest singular value of $Q$ by $\sigma_r$. If $f(\cdot,\cdot)$ is the inner product, then  $\sigma_r=1$. In general, if $Q$ is nonsingular, $\sigma_r$ is a positive constant. We can bound the subspace estimation errors, defined by $\norm{\tU^T\hU}_\rmF$ and $\norm{\tV^T\hV}_\rmF$.\footnote{The subspace estimation error $\|\tU^T\hU\|_\rmF = \|\hU-P_{U}\hU\|_\rmF$ evaluates the residual of $\hU$ when projected onto the subspace encoded by $U$. Clearly, the estimation error is between $0$ and $\sqrt{r}$, attaining $0$ when $\hU$ and $U$ span the same subspace, and attaining $\sqrt{r}$ when the two subspaces are orthogonal.} We have the following corollary:
\begin{corollary} \label{cor:lin}
If $r = O(1)$ and $\sigma_r>0$, then
\begin{align*}
\max\left\{\bbE\left[\norm{\tU^T\hU}_\rmF\right],~ \bbE\left[\norm{\tV^T\hV}_\rmF\right]\right\}= O\left(\sqrt{\frac{n_1n_2}{m}} \right).
\end{align*}
\end{corollary}

By Corollary \ref{cor:lin}, we need $m =O(n_1n_2)$ measurements to produce an accurate estimate, which is not efficient when $n_1,n_2$ are large. We present solutions to this in the next two sections.

%%%%%%%%%%%%%%%%%%%%%%%%%%%%%% s1 s2 Var Selection %%%%%%%%%%%%%%%%%%%%%%%%%%%%%%
\section{Variable Selection}\label{sec:vs}
When the data dimension is large, to reduce redundancy, and to improve robustness and efficiency, it is common to select a smaller number of variables for regression. For the problem described in Section \ref{sec:nonlinear}, the output variable $y$ depends on the input variable $a,b$ only through $U^Ta,V^Tb$. We now assume that there are no more than $s_1$ (resp. $s_2$) nonzero rows in $U$ (resp. $V$), where $r<s_1<n_1$ and $r<s_2<n_2$. Therefore, only $s_1$ variables in $a$ and $s_2$ variables in $b$ are active, and they are each reduced to $r$ variables in $U^Ta$ and $V^Tb$, respectively. As far as we know, previous supervised dimensionality reduction approaches with variable selection use LASSO-type solvers, and have no guarantees for exact recovery or only partial guarantees \cite{Li2007,Chen2010}.

Let $\norm{\cdot}_0$ denote the number of nonzero entries in a vector or a matrix, and let $\norm{\cdot}_{0,r}$ and $\norm{\cdot}_{0,c}$ denote the numbers of nonzero rows and nonzero columns, respectively. Let $P_{\Omega} Y \defeq \underset{X\in\Omega}{\arg\min} \norm{X-Y}_F$ denote the projection of matrix $Y$ onto set $\Omega$. Define a few sets:
\begin{itemize}
	\item The set of matrices that have at most $s_1$ nonzero entries in each column: $\Omega_1 \defeq \{X\in\bbR^{n_1\times n_2}: \norm{X^{(:,k)}}_0\leq s_1,~\forall k\in [n_2]\}$.
	\item The set of matrices with at most $s_2$ nonzero columns: $\Omega_2 \defeq \{X\in\bbR^{n_1\times n_2}: \norm{X}_{0,c}\leq s_2\}$. 
	\item The set of matrices with at most $s_1$ nonzero rows: $\Omega_3 \defeq \{X\in\bbR^{n_1\times n_2}: \norm{X}_{0,r}\leq s_1\}$.
	\item The set of matrices of at most rank-$r$: $\Omega_r \defeq \{X\in\bbR^{n_1\times n_2}: \rank(X)\leq r\}$.
\end{itemize}
We use the following three-step procedure to estimate $U$ and $V$. % which have at most $s_1$ and $s_2$ nonzero rows, respectively.

\emph{Step 1.} Compute the linear estimate $\hX_\mathrm{lin}$ in \eqref{eq:lin}.
%\[
%\hX_\mathrm{lin} \defeq \frac{1}{m} \sum_{i=1}^{m} a_i y_i b_i^T.
%\]

\emph{Step 2.} Compute an $(s_1,s_2)$-sparse approximation, i.e., one that has $s_1$ nonzero rows and $s_2$ nonzero columns. We are not aware of a computationally tractable algorithm that finds the best $(s_1,s_2)$-sparse approximation of $\hX_\mathrm{lin}$. Therefore, we consider a suboptimal but efficient approximation, first proposed by Lee et al. \cite{Lee2013} for sparse rank-1 matrix recovery:

\emph{2.1.} Compute $\hX_1 \defeq P_{\Omega_1} \hX_\mathrm{lin}$ by setting to zero all but the $s_1$ largest entries in each column of $\hX_\mathrm{lin}$ (in terms of absolute value).

\emph{2.2.} Compute $\hX_2 \defeq P_{\Omega_2} \hX_1$ by setting to zero all but the $s_2$ largest columns in $\hX_1$ (in terms of $\ell_2$ norm).

\emph{2.3.} Compute $\hX_3 \defeq P_{\Omega_3} \hX_2$ by setting to zero all but the $s_1$ largest rows in $\hX_2$ (in terms of $\ell_2$ norm).
	
\emph{Step 3.} Compute a rank-$r$ approximation. By taking the SVD and keeping the $r$ largest singular values and singular vectors, we find the best rank-$r$ approximation $\hU'\hSigma'\hV'^T =  P_{\Omega_r}\hX_3$. Note that $\hX_3$ only has $s_1$ nonzero rows and $s_2$ nonzero columns, hence computing its SVD is much cheaper than computing the SVD of dense matrix $\hX_\mathrm{lin}$.

This estimator is a sequential projection: $\hU'\hSigma'\hV'^T = P_{\Omega_r} P_{\Omega_3}P_{\Omega_2} P_{\Omega_1} \hX_\mathrm{lin}$, and satisfies $\hU'\hSigma'\hV'^T \in \Omega_2\bigcap\Omega_3\bigcap\Omega_r$.
Next, we bound the error of this estimator. In particular, we show that the nonlinear estimator $\hX_2$ has a much smaller error than the linear estimator $\hX_\mathrm{lin}$ (Theorem \ref{thm:nonlin}), and $\hU'\hSigma'\hV'^T$ is almost as good as $\hX_2$ (Corollary \ref{cor:nonlin}).

\begin{theorem}\label{thm:nonlin}
For $n_1,n_2 \geq 8$,
\begin{align*}
\bbE\left[\norm{\hX_2 - X}_\rmF\right] \leq &~ 2\sqrt{\frac{r^2\sigma_{y|a,b}^2+\sigma^2}{m}} + 8\sqrt{\frac{2s_1s_2\log n_1\log n_2 \cdot (\sigma_{y|a,b}^2+\tau_0^2)}{m}} \\
&~ + 4\sqrt{\frac{2s_1s_2\log n_2 \cdot (r\sigma_{y|a,b}^2+\tau_1^2)}{m}} + 4\sqrt{\frac{2s_1s_2\log n_1 \cdot(r\sigma_{y|a,b}^2+\tau_2^2)}{m}}.
\end{align*}
\end{theorem}

\begin{corollary}\label{cor:nonlin}
If $r = O(1)$ and $\sigma_r>0$, then
\begin{align*}
\max\left\{\bbE\left[\norm{\tU^T\hU'}_\rmF\right],~\bbE\left[\norm{\tV^T\hV'}_\rmF\right] \right\} = O\left(\sqrt{\frac{s_1s_2\log n_1 \log n_2}{m}}\right).
\end{align*}
\end{corollary}

Corollary \ref{cor:nonlin} yields a sample complexity $m =O(s_1s_2\log n_1 \log n_2)$ that is much less demanding than the one without variable selection.

%%%%%%%%%%%%%%%%%%%%%%%%%%%%%% Incoherence %%%%%%%%%%%%%%%%%%%%%%%%%%%%%%
\section{Optimal Sample Complexity} \label{sec:optimal}
Careful readers may have noticed that the number of degrees of freedom of $U,V$ in Section \ref{sec:dr} is $O(n_1+n_2)$. Hence the sample complexity $O(n_1n_2)$ is suboptimal. In this section, we show that near optimal sample complexity (sample complexity that is optimal up to constants and log factors) can be achieved when $\{y_i\}_i^{m}$ are i.i.d. light-tailed random variables, i.e., there exists global constants $c,C>0$, s.t.
\begin{align}
\bbP\left[\left|y_i\right|\geq t\right] \leq C e^{-ct},\quad \forall t\geq 0. \label{eq:light}
\end{align}
We call this mild condition the light-tailed measurement condition. Please refer to Section \ref{sec:exp} for examples that satisfy the light-tailed measurement condition.

%\subsection{Optimal Bound for Dimensionality Reduction}
In Section \ref{sec:dr}, inequality \eqref{eq:rankr} shows that $\hU\hSigma\hV^T = P_{\Omega_r} \hX_\mathrm{lin}$, as the best rank-$r$ approximation of $\hX_\mathrm{lin}$, is almost as good as $\hX_\mathrm{lin}$. Next, Theorem \ref{thm:optimal} shows that, under the light-tailed measurement condition, $\hU\hSigma\hV^T$ is significantly better than $\hX_\mathrm{lin}$. \footnote{As an interesting side note, the light-tailed measurement condition is similar in spirit to the ``spectral flatness'' condition in blind deconvolution \cite{Lee2015a}. Under the light-tailed measurement condition, $\max_i |y_i| = O(\log m)$, which is analogous to the bounds established in \cite[Propositions 2.1 -- 2.3]{Lee2015a}. However, the approaches of \cite{Lee2015a} and this paper are quite different.}

\begin{theorem}\label{thm:optimal}
Suppose $\{y_i\}_{i=1}^m$ are i.i.d. light-tailed random variables defined by \eqref{eq:light}, where $C>0$ and $c>\frac{1}{8\log(n_1+n_2)}$. If $m>n_1+n_2$, then
\begin{align*}
\bbE\left[\norm{\hU\hSigma\hV^T - X}_\rmF\right] \leq &~ 2\sqrt{\frac{r^2\sigma_{y|a,b}^2+\sigma^2}{m}} + 512\sqrt{2}(C+2) \sqrt{\frac{(n_1+n_2)r\log^2m\log^4(n_1+n_2)}{m}} \\
&~ + 2\sqrt{\frac{n_2 (r\sigma_{y|a,b}^2+\tau_1^2)}{m}} + 2\sqrt{\frac{n_1 (r\sigma_{y|a,b}^2+\tau_2^2)}{m}}.
\end{align*}
\end{theorem}
\begin{corollary}\label{cor:optimal}
If $r=O(1)$ and $\sigma_r>0$, then under the same conditions as in Theorem \ref{thm:optimal},
\begin{align*}
\max\left\{\bbE\left[\norm{\tU^T\hU}_\rmF\right],~\bbE\left[\norm{\tV^T\hV}_\rmF\right] \right\} = O\left(\sqrt{\frac{(n_1+n_2)\log^2 m \log^4(n_1+n_2)}{m}}\right).
\end{align*}
\end{corollary}
Under the light-tailed measurement condition, projection onto the set of rank-$r$ matrices significantly reduces the error in the linear estimator $\hX_\mathrm{lin}$. In this case, we only need $m=O((n_1+n_2)\log^6(n_1+n_2))$ samples to obtain an accurate estimate, as opposed to $m=O(n_1n_2)$.

%%%%%%%%%%%%%%%%%%%%%%%%%%%%%% Practical Considerations %%%%%%%%%%%%%%%%%%%%%%%%%%%%%%
\section{Discussions and Experiments}

\subsection{Generalization of the Model}
Throughout Sections \ref{sec:dr} -- \ref{sec:optimal}, we assume that: 1) $\{a_i\}_{i=1}^m$ and $\{b_i\}_{i=1}^m$ are independent random vectors, following Gaussian distributions $N(0,I_{n_1})$ and $N(0,I_{n_2})$, respectively; 2) $U$ and $V$ have orthonormal columns. These assumptions can be easily relaxed. Suppose $U\in\bbR^{n_1\times r}$ and $V\in\bbR^{n_2\times r}$ are tall matrices of full column rank, but may not have orthonormal columns. Suppose $\{a_i\}_{i=1}^m$ and $\{b_i\}_{i=1}^m$ are independent random vectors, following Gaussian distributions $N(\mu_1,\Sigma_1)$ and $N(\mu_2,\Sigma_2)$, respectively, and $\mu_1,\mu_2,\Sigma_1,\Sigma_2$ are known, or can be estimated before hand. Let $\Sigma_1 = C_1C_1^T$ and $\Sigma_2 = C_2C_2^T$ denote the Cholesky decompositions of the covariance matrices. Then by a simple change of variables $a'_i = C_1^{-1} (a_i-\mu_1)$ and $b'_i = C_2^{-1} (b_i-\mu_2)$, the embeddings $U'$ and $V'$ estimated from $a'_i$, $b'_i$, and $y_i$ are orthogonal bases for the column spaces of $C_1^T U$ and $C_2^T V$. If $U$ (resp. $V$) has $s_1$ (resp. $s_2$) nonzero rows, then the columns of $U'$ (resp. $V'$) are jointly $s_1$ (resp. $s_2$) sparse over ``dictionary'' $C_1^T$ (resp. $C_2^T$). Provided that the condition numbers of $C_1$, $C_2$ (or $\Sigma_1$, $\Sigma_2$) are bounded by a constant independent of $n_1$ and $n_2$, the previous analysis translates to this scenario with virtually no change.

%Next, we demonstrate that, after a few manipulations, this scenario reduces to a problem that satisfies (A1) and (A2).
%\begin{enumerate}
	%\item Let $\Sigma_1 = C_1C_1^T$ and $\Sigma_2 = C_2C_2^T$ denote the Cholesky decompositions of the covariance matrices.
	%\item Let  $a'_i = C_1^{-1} (a_i-\mu_1)$ and $b'_i = C_2^{-1} (b_i-\mu_2)$, which clearly follow standard Gaussian distributions $N(0,I_{n_1})$ and $N(0,I_{n_2})$, respectively.
	%\item Let the columns of $U'$ and $V'$ be the left singular vectors of $C_1^T U$ and $C_2^T V$, respectively, and  $U'Q_1 = C_1^T U$ and $V'Q_2 = C_2^T V$.
	%\item Let
	%\begin{align*}
	%f(U^T a_i,~V^T b_i) = &~ f\left(U^T (a_i-\mu_1)+U^T\mu_1,~V^T (b_i-\mu_2)+V^T\mu_2\right)\\
	%= &~ f\left(Q_1^T U'^T a'_i+U^T\mu_1,~Q_2^T V'^T b'_i+V^T\mu_2\right)\\
	%\eqdef &~ f'( U'^T a'_i,~V'^T b'_i).
	%\end{align*}
%\end{enumerate}
%Clearly, the output $y_i$ depends on $a'_i, b'_i$ only through $U'^T a'_i$ and $V'^T a'_i$, in which $U'$ and $V'$ encode the dimensionality reduction subspaces. If $U$ (resp. $V$) has $s_1$ (resp. $s_2$) nonzero rows, then the columns of $U'$ (resp. $V'$) are jointly $s_1$ (resp. $s_2$) sparse over ``dictionary'' $C_1^T$ (resp. $C_2^T$). Provided that the condition numbers of $C_1$, $C_2$ (or $\Sigma_1$, $\Sigma_2$) are bounded by a constant independent of $n_1$ and $n_2$, the previous analysis translates to this scenario with virtually no change.

When the means and variance matrices of $a$ and $b$ are unknown, sample means and sample covariance matrices can be used in practice. Numerical experiments in Section \ref{sec:exp} show that using sample means and covariances causes no significant change in estimation accuracy. Instead of the explicit error bounds in Sections \ref{sec:dr} -- \ref{sec:optimal}, one can show asymptotic bounds, where the penalty for using sample means and covariance matrices estimated from $m$ samples is an extra term of $O(1/\sqrt{m})$ in the error bounds.

The Gaussianity and independence assumptions are crucial to the theoretical analysis of our joint dimensionality reduction approach. However, numerical experiments in Section \ref{sec:exp} confirm that our approach can estimate the embeddings accurately when the distributions are non-Gaussian (e.g., uniform, Poisson) or there are weak dependencies between $a$ and $b$.
Previous supervised dimensionality reduction approaches (SIR, pHd, SAVE) can be extended to non-Gaussian distributions that satisfy certain properties (linear conditional mean, constant conditional variance, etc.). We conjecture that the same extension applies also to our approach.

\subsection{Bivariate Nonlinear Functional}\label{sec:bivariate}
The simple bilinear regression model $f(\ba_i,\bb_i) = \ba_i^T \bb_i$ is a motivating application of this paper. In this case, $Q = I_r$, $\sigma_r=1$, $\sigma^2 = r(r+2)^2 -r$, $\tau_0^2 = r$, $\tau_1^2 = \tau_2^2 = r(r+2)$. By Theorems \ref{thm:lin}, \ref{thm:nonlin}, and \ref{thm:optimal}, we can derive explicit error bounds in terms of $m$, $n_1$, $n_2$, $s_1$, $s_2$, and $r$. For example, if $\sigma_{y|a,b}^2 = O(r)$ and $r = O(\min\{\log n_1, \log n_2\})$, then under the assumptions in Sections \ref{sec:dr}, \ref{sec:vs}, and \ref{sec:optimal}, the normalized error $\max\left\{\bbE\left[\frac{1}{\sqrt{r}}\|\tU^T\hU\|_\rmF\right],~ \bbE\left[\frac{1}{\sqrt{r}}\|\tV^T\hV\|_\rmF\right]\right\}$ is bounded by $O\left(\sqrt{\frac{n_1n_2}{m}} \right)$, $O\left(\sqrt{\frac{s_1s_2\log n_1 \log n_2}{m}}\right)$, and $O\left(\sqrt{\frac{(n_1+n_2)\log^2 m \log^4(n_1+n_2)}{m}}\right)$, respectively. Therefore, for bilinear regression model with reasonable signal to noise ratio, the aforementioned sample complexities hold for $r = O(\min\{\log n_1, \log n_2\})$, as opposed to just $r=O(1)$.

In general, the nonlinear functional $f(\cdot,\cdot)$ can take many forms (e.g., Mercer kernels, neural networks with two inputs), and can be unknown beforehand. %Fortunately, none of the estimators requires the knowledge of $f(\cdot,\cdot)$ or $Q$. The constants in the error bounds ($\sigma_r,\sigma,\tau_0,\tau_1,\tau_2,\sigma_{y|a,b},c$, and $C$) either are global constants, or only depend on $r$. When $r=O(1)$, we have a clear idea of how the estimators behave without the knowledge of $f(\cdot,\cdot)$.
However, we do need $Q = \bbE\left[\ba_1 f(\ba_1,\bb_1) \bb_1^T \right]$ to be nonsingular. Clearly, our approach fails when $f(\ba_1,\bb_1)$ is even in $\ba_1$ or $\bb_1$. This is an intrinsic limitation of supervised dimension reduction. In fact, SIR, pHd, and SAVE all require similar assumptions, and they fail when $f(\cdot,\cdot)$ is odd in both variables. In this sense, our approach complements the previous supervised dimensionality reduction approaches. In Section \ref{sec:exp}, we present examples for which pHd fails and our approach succeeds, and vice versa.

%If $Q$ is singular but nonzero, e.g., $\rank(Q)=r_0<r$, then we can estimate $r_0$-dimensional subspaces of the original $r$-dimensional subspaces spanned by the columns of $U,V$. Partial estimates of the linear embeddings may still be useful, for example, in determining the active variables (feature selection), and in providing an initialization for other algorithms.

%If $f(U^Ta_1,V^Tb_1)$ is linear in $U^Ta_1,V^Tb_1$, meaning there exist $\alpha,\beta\in\bbR^r$ such that $f(\ba_1,\bb_1)=\alpha^TU^Ta_1+\beta^TV^Tb_1$, then $Q = 0$. In this case, we can not estimate the subspaces encoded by $U,V$ using the proposed approach. Instead, this problem reduces to linear regression, which has been addressed by \cite{Plan2014}.

\subsection{Estimation of Rank and Sparsity}
Throughout the paper, we assume that the rank $r$ and sparsity levels $s_1,s_2$ are known. In practice, these parameters often need to be estimated from data. We give a partial solution in this section.

If the sample complexity satisfies $m=\Omega(n_1n_2)$, then $r$, $s_1$, $s_2$ can be estimated from $\hX_\mathrm{lin}$ as follows. Let $(J,K)$ and $(J,K)^c$ denote the support of $X=UQV^T$ (the set of indices where $X$ is nonzero) and its complement. Let $\sigma_i(\cdot)$ denote the $i$-th singular value of a matrix. Suppose for some $\eta>0$,
\[
\min_{(j,k)\in(J,K)}|X^{(j,k)}|\geq \eta, \quad \sigma_r(X) = \sigma_r(Q) \geq \eta.
\]
By Theorem \ref{thm:lin}, we can achieve $\|\hX_\mathrm{lin}-X\|_\rmF \leq \frac{1}{3}\eta$ with $m=\Omega(n_1n_2)$ samples. Then
%\begin{align*}
%&\min_{(j,k)\in(J,K)}|\hX_\mathrm{lin}^{(j,k)}| \geq \min_{(j,k)\in(J,K)}|X^{(j,k)}| - \norm{\vect(\hX_\mathrm{lin}-X)}_\infty \geq \min_{(j,k)\in(J,K)}|X^{(j,k)}|- \norm{\hX_\mathrm{lin}-X}_\rmF \geq \frac{2}{3}\eta,\\
%&\max_{(j,k)\in(J,K)^c}|\hX_\mathrm{lin}^{(j,k)}| \leq \norm{\vect(\hX_\mathrm{lin}-X)}_\infty \leq \norm{\hX_\mathrm{lin}-X}_\rmF \leq \frac{1}{3}\eta,\\
%&\sigma_r(\hX_\mathrm{lin}) \geq \sigma_r(X)-\norm{\hX_\mathrm{lin}-X}\geq \sigma_r(X)-\norm{\hX_\mathrm{lin}-X}_\rmF \geq \frac{2}{3}\eta,\\
%&\sigma_{r+1}(\hX_\mathrm{lin}) \leq \norm{\hX_\mathrm{lin}-X} \leq \norm{\hX_\mathrm{lin}-X}_\rmF \leq \frac{1}{3}\eta.
%\end{align*}
\begin{align*}
\min_{(j,k)\in(J,K)}|\hX_\mathrm{lin}^{(j,k)}| \geq \frac{2}{3}\eta, \quad \max_{(j,k)\in(J,K)^c}|\hX_\mathrm{lin}^{(j,k)}| \leq \frac{1}{3}\eta,\quad
\sigma_r(\hX_\mathrm{lin}) \geq \frac{2}{3}\eta, \quad \sigma_{r+1}(\hX_\mathrm{lin}) \leq \frac{1}{3}\eta.
\end{align*}
Therefore, an entry is nonzero in $X$ if and only if the absolute value of the corresponding entry in $\hX_\mathrm{lin}$ is greater than $\frac{1}{2}\eta$. We can determine $s_1$ and $s_2$ by counting the number of such entries. Similarly, the rank $r$ of matrix $X$ can be determined by counting the number of singular values of $\hX_\mathrm{lin}$ greater than $\frac{1}{2}\eta$. In practice, such a threshold $\eta$ is generally unavailable. However, by gathering a sufficiently large number of samples, the entries and singular values of $\hX_\mathrm{lin}$ will vanish if the corresponding entries and singular values in $X$ are zero. 

Li \cite{Li1991,Li1992} derived $\chi^2$ tests to assess the true dimension $r$ of the embedding in SIR and pHd. We expect similar tests can be derived for our approach.

%When we have less than $n_1n_2$ samples (e.g., the sample complexities $O(s_1s_2\log n_1\log n_2)$ and $O((n_1+n_2)\log^6(n_1+n_2))$ in Sections \ref{sec:vs} and \ref{sec:optimal}), it is unclear how to obtain a guarantee for the estimation of $r,s_1,s_2$. If $f(\cdot,\cdot)$ can be estimated given guesses of $r,s_1,s_2$ and estimates of $U,V$, we can choose $r,s_1,s_2$ by cross validation. 

%\begin{remark}
%We only discuss regression with two predictor vectors. The same analysis translates to regression with multiple predictor vectors, simply by replacing matrices, outer products, and SVD, with tensors, tensor products, and Tucker decomposition.
%\end{remark}

%Using various concentration of measure inequalities (e.g., Chernoff bound, Bernstein inequality), one can easily derive the ``high probability'' versions of the error bounds in this paper, i.e., the error bound is valid with high probability rather than in expectation.

\subsection{Experiments}\label{sec:exp}
In this section, we verify our theoretical analysis with some numerical experiments. Here, the normalized subspace estimation error (NSEE) is defined by $\max\left\{\frac{1}{\sqrt{r}}\|\tU^T\hU\|_\rmF,~ \frac{1}{\sqrt{r}}\|\tV^T\hV\|_\rmF\right\}$.

First, we test the estimators $\hU,\hV$ (Sections \ref{sec:dr} and \ref{sec:optimal}) and $\hU',\hV'$ (Section \ref{sec:vs}) on two different models, dubbed \texttt{BILINEAR} and \texttt{BINARY}, both of which satisfy the light-tailed measurement condition:
\begin{itemize}
	\item Bilinear regression with additive Gaussian noise. Let $\mu_i = f(U^Ta_i,V^Tb_i) = a_i^TUV^T b_i$, and $y_i = \mu_i+z_i$, where $\{z_i\}_{i=1}^m$ are i.i.d. Gaussian random variables $N(0,1)$. %By simple calculation, $M_{y_i}(t) = (1-t^2)^{-\frac{r}{2}}e^{\frac{t^2}{2}}$ for $t\in(-1,1)$, which is finite in this interval.
	\item Logistic-type binary classification. Let $\mu_i = f(U^Ta_i,V^Tb_i) = \exp(-\|U^T a_i-V^T b_i\|_2^2)$, and $y_i\sim\operatorname{Ber}(\mu_i)$ is a Bernoulli random variable with mean $\mu_i$. %If we marginalize out $\ba_i,\bb_i$ for $i=1,2,\cdots,m$, then $\{y_i\}_{i=1}^m$ are i.i.d. Bernoulli random variables $y_i\sim\operatorname{Ber}(\mu)$, where $\mu = \bbE[\mu_i]$. The corresponding moment generating function is $M_{y_i}(t) = 1-\mu+\mu e^t$.
\end{itemize}
Let $n_1=n_2=n$ and $s_1=s_2=s$. For each model, we conduct four experiments. Without variable selection, we fix $n$ (resp. $m$) and study how error varies with $m$ (resp. $n$). With variable selection, we fix $n,s$ (resp. $n,m$) and study how error varies with $m$ (resp. $s$). We repeat every experiments $100$ times, and show in Figure \ref{fig:log-log} the log-log plot of the mean error versus $m,n$ or $s$. The results for the two models are roughly the same, which verifies that our algorithm and theory apply to different regression problems. Nonlinearity in the model determines only the constants in the error bounds. The slopes of the plots in the first and third columns are roughly $-0.5$, which verifies the term $O(1/\sqrt{m})$ in the error bounds. The slopes of the plots in the second column are roughly $0.5$, which verifies the term $O(\sqrt{n_1+n_2})=O(\sqrt{n})$ in the error bound in Theorem \ref{thm:optimal}. The slopes of the plots in the fourth column are roughly $1$, which verifies the term $O(\sqrt{s_1s_2})=O(s)$ in the error bound in Theorem \ref{thm:nonlin}. 

\begin{figure}[htbp]
\centering
\begin{tikzpicture}[scale = 0.45]

\begin{axis}[%
width=2.282in,
height=1.8in,
at={(0in,0in)},
scale only axis,
xmin=6.90775527898214,
xmax=9.21034037197618,
xlabel={$\log(m)$},
ymin=-1.65133908669857,
ymax=0.164732059227385,
ylabel={$\log(\mathrm{NSEE})$},
axis background/.style={fill=white}
]
\addplot [color=blue,solid,line width=2.0pt,forget plot]
  table[row sep=crcr]{%
6.90775527898214	-0.215194441711561\\
7.60090245954208	-0.488714027390471\\
8.51719319141624	-0.929333134724102\\
9.21034037197618	-1.27141258575962\\
};
\end{axis}
\end{tikzpicture}%
\begin{tikzpicture}[scale = 0.45]

\begin{axis}[%
width=2.282in,
height=1.8in,
at={(0in,0in)},
scale only axis,
xmin=3.91202300542815,
xmax=6.21460809842219,
xlabel={$\log(n)$},
ymin=-1.96308556601509,
ymax=-0.147014420089143,
ylabel={$\log(\mathrm{NSEE})$},
axis background/.style={fill=white}
]
\addplot [color=blue,solid,line width=2.0pt,forget plot]
  table[row sep=crcr]{%
3.91202300542815	-1.63435455894625\\
4.60517018598809	-1.2740427772473\\
5.29831736654804	-0.923316299740511\\
6.21460809842219	-0.475745427157983\\
};
\end{axis}
\end{tikzpicture}%
\begin{tikzpicture}[scale = 0.45]

\begin{axis}[%
width=2.282in,
height=1.8in,
at={(0in,0in)},
scale only axis,
xmin=9.21034037197618,
xmax=11.5129254649702,
xlabel={$\log(m)$},
ymin=-4.08349742591548,
ymax=-2.26742627998953,
ylabel={$\log(\mathrm{NSEE})$},
axis background/.style={fill=white}
]
\addplot [color=blue,solid,line width=2.0pt,forget plot]
  table[row sep=crcr]{%
9.21034037197618	-2.55955806313597\\
9.90348755253613	-2.93485215962578\\
10.8197782844103	-3.39577582809748\\
11.5129254649702	-3.79136564276904\\
};
\end{axis}
\end{tikzpicture}%
\begin{tikzpicture}[scale = 0.45]

\begin{axis}[%
width=2.282in,
height=1.8in,
at={(0in,0in)},
scale only axis,
xmin=2.23694911729812,
xmax=4.67080616168402,
xlabel={$\log(s)$},
ymin=-3.33684670266371,
ymax=-1.41724009830129,
ylabel={$\log(\mathrm{NSEE})$},
axis background/.style={fill=white}
]
\addplot [color=blue,solid,line width=2.0pt,forget plot]
  table[row sep=crcr]{%
2.30258509299405	-3.33684670266371\\
2.99573227355399	-2.68991573176776\\
3.91202300542815	-2.0295046437283\\
4.60517018598809	-1.41724009830129\\
};
\end{axis}
\end{tikzpicture}%
\\
\begin{tikzpicture}[scale = 0.45]

\begin{axis}[%
width=2.282in,
height=1.8in,
at={(0in,0in)},
scale only axis,
xmin=9.90348755253613,
xmax=12.2060726455302,
xlabel={$\log(m)$},
ymin=-3.82992250251224,
ymax=-2.01385135658629,
ylabel={$\log(\mathrm{NSEE})$},
axis background/.style={fill=white}
]
\addplot [color=blue,solid,line width=2.0pt,forget plot]
  table[row sep=crcr]{%
9.90348755253613	-2.34853472496215\\
10.8197782844103	-2.8173177878998\\
11.5129254649702	-3.14841200116734\\
12.2060726455302	-3.49523913413638\\
};
\end{axis}
\end{tikzpicture}%
\begin{tikzpicture}[scale = 0.45]

\begin{axis}[%
width=2.282in,
height=1.8in,
at={(0in,0in)},
scale only axis,
xmin=2.30258509299405,
xmax=4.60517018598809,
xlabel={$\log(n)$},
ymin=-3.47200059299522,
ymax=-1.65592944706927,
ylabel={$\log(\mathrm{NSEE})$},
axis background/.style={fill=white}
]
\addplot [color=blue,solid,line width=2.0pt,forget plot]
  table[row sep=crcr]{%
2.30258509299405	-3.18594329230943\\
2.99573227355399	-2.78266797880402\\
3.91202300542815	-2.3097815954959\\
4.60517018598809	-1.94198674775507\\
};
\end{axis}
\end{tikzpicture}%
\begin{tikzpicture}[scale = 0.45]

\begin{axis}[%
width=2.282in,
height=1.8in,
at={(0in,0in)},
scale only axis,
xmin=12.2060726455302,
xmax=14.5086577385242,
xlabel={$\log(m)$},
ymin=-4.80549659137636,
ymax=-2.98942544545042,
ylabel={$\log(\mathrm{NSEE})$},
axis background/.style={fill=white}
]
\addplot [color=blue,solid,line width=2.0pt,forget plot]
  table[row sep=crcr]{%
12.2060726455302	-3.1982945629925\\
13.1223633774043	-3.79488492192837\\
13.8155105579643	-4.14964904055467\\
14.5086577385242	-4.59662747383429\\
};
\end{axis}
\end{tikzpicture}%
\begin{tikzpicture}[scale = 0.45]

\begin{axis}[%
width=2.282in,
height=1.8in,
at={(0in,0in)},
scale only axis,
xmin=1.49252226458343,
xmax=4.02893865327882,
xlabel={$\log(s)$},
ymin=-4.6840019801908,
ymax=-2.6835058284617,
ylabel={$\log(\mathrm{NSEE})$},
axis background/.style={fill=white}
]
\addplot [color=blue,solid,line width=2.0pt,forget plot]
  table[row sep=crcr]{%
1.6094379124341	-4.6840019801908\\
2.30258509299405	-4.17287904031326\\
2.99573227355399	-3.4981562412418\\
3.91202300542815	-2.6835058284617\\
};
\end{axis}
\end{tikzpicture}%
\caption{Log-log plots of mean error versus $m,n$ or $s$. The two rows are plots for the two models, \texttt{BILINEAR} and \texttt{BINARY}. Within each row, the four plots correspond to the four experiments.}
\label{fig:log-log}
\end{figure}
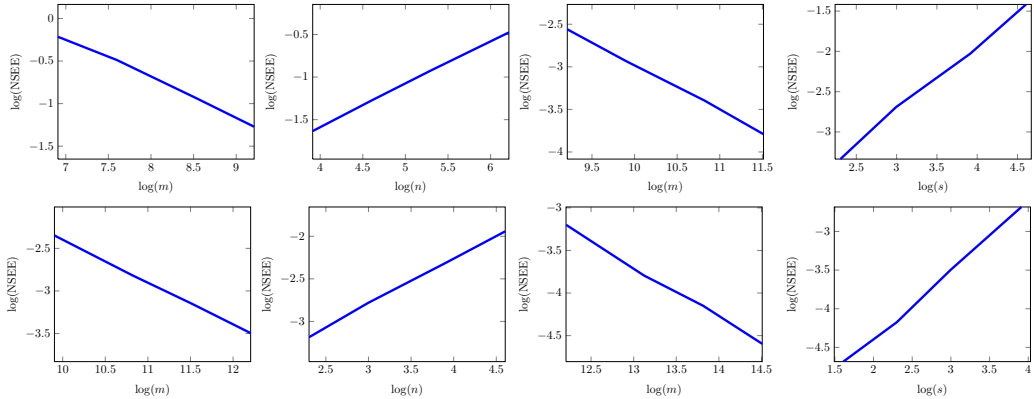

Next, we test how our estimator performs when the assumptions are violated, i.e., when 1) the true means and variances are replaced by sample means and variances, or 2) the entries of $a_i,b_i$ are i.i.d. following a uniform distribution on $[-\sqrt{3},\sqrt{3}]$, or 3) the entries of $a_i,b_i$ are i.i.d. following a Poisson distribution ($\lambda = 4$, normalized with zero mean and unit variance), or 4) $a_i,b_i$ are jointly Gaussian and weakly correlated (not independent). Clearly, there is no significant change in the performance.

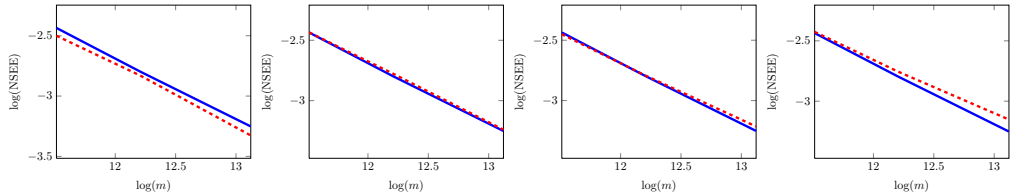
\begin{figure}[htbp]
\centering
\begin{tikzpicture}[scale=0.45]

\begin{axis}[%
width=2.26in,
height=1.783in,
at={(0in,0in)},
scale only axis,
xmin=11.5129254649702,
xmax=13.1223633774043,
xlabel={$\log(m)$},
ymin=-3.51575278132376,
ymax=-2.24637352458138,
ylabel={$\log(\mathrm{NSEE})$},
axis background/.style={fill=white}
]
\addplot [color=blue,solid,line width=2.0pt,forget plot]
  table[row sep=crcr]{%
11.5129254649702	-2.43563598165736\\
12.2060726455302	-2.79689097157546\\
13.1223633774043	-3.25176321291227\\
};
\addplot [color=red,dashed,line width=2.0pt,forget plot]
  table[row sep=crcr]{%
11.5129254649702	-2.50058616160947\\
12.2060726455302	-2.82956987009585\\
13.1223633774043	-3.32649032424779\\
};
\end{axis}
\end{tikzpicture}%
\begin{tikzpicture}[scale=0.45]

\begin{axis}[%
width=2.26in,
height=1.783in,
at={(0in,0in)},
scale only axis,
xmin=11.5129254649702,
xmax=13.1223633774043,
xlabel={$\log(m)$},
ymin=-3.47801872363627,
ymax=-2.20863946689389,
ylabel={$\log(\mathrm{NSEE})$},
axis background/.style={fill=white}
]
\addplot [color=blue,solid,line width=2.0pt,forget plot]
  table[row sep=crcr]{%
11.5129254649702	-2.43563598165736\\
12.2060726455302	-2.79689097157546\\
13.1223633774043	-3.25176321291227\\
};
\addplot [color=red,dashed,line width=2.0pt,forget plot]
  table[row sep=crcr]{%
11.5129254649702	-2.43489497761789\\
12.2060726455302	-2.77584270005024\\
13.1223633774043	-3.24120683231683\\
};
\end{axis}
\end{tikzpicture}%
\begin{tikzpicture}[scale=0.45]

\begin{axis}[%
width=2.26in,
height=1.783in,
at={(0in,0in)},
scale only axis,
xmin=11.5129254649702,
xmax=13.1223633774043,
xlabel={$\log(m)$},
ymin=-3.478389225656,
ymax=-2.20900996891362,
ylabel={$\log(\mathrm{NSEE})$},
axis background/.style={fill=white}
]
\addplot [color=blue,solid,line width=2.0pt,forget plot]
  table[row sep=crcr]{%
11.5129254649702	-2.43563598165736\\
12.2060726455302	-2.79689097157546\\
13.1223633774043	-3.25176321291227\\
};
\addplot [color=red,dashed,line width=2.0pt,forget plot]
  table[row sep=crcr]{%
11.5129254649702	-2.45040266370146\\
12.2060726455302	-2.79241663529589\\
13.1223633774043	-3.21464610661094\\
};
\end{axis}
\end{tikzpicture}%
\begin{tikzpicture}[scale=0.45]

\begin{axis}[%
width=2.26in,
height=1.783in,
at={(0in,0in)},
scale only axis,
xmin=11.5129254649702,
xmax=13.1223633774043,
xlabel={$\log(m)$},
ymin=-3.47333476862539,
ymax=-2.20395551188301,
ylabel={$\log(\mathrm{NSEE})$},
axis background/.style={fill=white}
]
\addplot [color=blue,solid,line width=2.0pt,forget plot]
  table[row sep=crcr]{%
11.5129254649702	-2.43563598165736\\
12.2060726455302	-2.79689097157546\\
13.1223633774043	-3.25176321291227\\
};
\addplot [color=red,dashed,line width=2.0pt,forget plot]
  table[row sep=crcr]{%
11.5129254649702	-2.42552706759612\\
12.2060726455302	-2.75601536346579\\
13.1223633774043	-3.15223244932912\\
};
\end{axis}
\end{tikzpicture}%
\caption{Log-log plots when the assumptions are violated. The blue solid lines are the performances when all the assumptions are met. The red dashed lines are the performances when the assumptions are violated in four different ways: 1) sample means and variances are used; 2) $a_i,b_i$ follow uniform distribution; 3) $a_i,b_i$ follow Poisson distribution; 4) $a_i$ and $b_i$ are weakly correlated.
}
\label{fig:violate}
\end{figure}

In the last experiment (see Figure \ref{fig:compare_pHd}), we compare our approach with principal Hessian direction (pHd) for two link functions: 1) $f(\ba_i,\bb_i) = \ba_i^T\bb_i = \sum_{j=1}^r \ba_i^{(j)}\bb_i^{(j)}$, which is odd in $\ba_i,\bb_i$, and 2) $f(\ba_i,\bb_i) = \sum_{j=1}^r \ba_i^{(j)2}\bb_i^{(j)2}$, which is even in $\ba_i,\bb_i$. For the odd function, our approach succeeds, but pHd fails. For the even function, our approach fails, but pHd succeeds.

\begin{figure}[htbp]
\centering
\begin{tikzpicture}[scale=0.45]

\begin{axis}[%
width=1.465in,
height=1.783in,
at={(0in,0in)},
scale only axis,
xmin=6.908,
xmax=9.21,
xlabel={$\log(m)$},
ymin=-2.8,
ymax=0,
ylabel={$\log(\mathrm{NSEE})$},
axis background/.style={fill=white}
]
\addplot [color=blue,solid,line width=2.0pt,forget plot]
  table[row sep=crcr]{%
6.90775527898214	-1.58020068553952\\
7.60090245954208	-1.90241849989738\\
8.51719319141624	-2.39339524848198\\
9.21034037197618	-2.73179799885045\\
};
\addplot [color=red,dashed,line width=2.0pt,forget plot]
  table[row sep=crcr]{%
6.90775527898214	-0.235154733757411\\
7.60090245954208	-0.212884630672413\\
8.51719319141624	-0.245352264966206\\
9.21034037197618	-0.246449811460298\\
};
\end{axis}
\end{tikzpicture}%
~~~~
\begin{tikzpicture}[scale=0.45]

\begin{axis}[%
width=1.465in,
height=1.783in,
at={(0in,0in)},
scale only axis,
xmin=6.908,
xmax=9.21,
xlabel={$\log(m)$},
ymin=-2.8,
ymax=0,
ylabel={$\log(\mathrm{NSEE})$},
axis background/.style={fill=white}
]
\addplot [color=blue,solid,line width=2.0pt,forget plot]
  table[row sep=crcr]{%
6.90775527898214	-0.357542349616585\\
7.60090245954208	-0.345415705555318\\
8.51719319141624	-0.35742561210398\\
9.21034037197618	-0.409221455986557\\
};
\addplot [color=red,dashed,line width=2.0pt,forget plot]
  table[row sep=crcr]{%
6.90775527898214	-1.1528123149822\\
7.60090245954208	-1.50108420216452\\
8.51719319141624	-2.0156654517911\\
9.21034037197618	-2.41665463667388\\
};
\end{axis}
\end{tikzpicture}%
\caption{Log-log plots of our approach (blue solid lines) versus pHd (red dashed lines). The left plot is for an odd function, and the right plot is for an even function.
}
\label{fig:compare_pHd}
\end{figure}
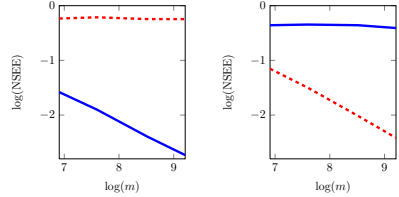

%\subsection{Open Problems}
%We assume that the input variables $\{a_i\}_{i=1}^m$ and $\{b_i\}_{i=1}^m$ follow Gaussian distributions. The error bounds of the estimators for non-Gaussian input are interesting open questions.
%
%For variable selection and dimensionality reduction in Section \ref{sec:vs}, we showed that $O(s_1s_2\log n_1\log n_2)$ samples are sufficient for an accurate estimate. However, the number of degrees of freedom in this case is $O(s_1+s_2)$. It is of interest to derive a near optimal sample complexity bound that matches the number of degrees of freedom up to log factors.

%\section*{Acknowledgments}
%This work was supported in part by the National Science Foundation (NSF) under Grant IIS 14-47879.

%\bibliographystyle{myIEEEtran}
%\bibliography{IEEEabrv,nips2016}

\newpage

%%%%%%%%%%%%%%%%%%%%%%%%%%%%%% Proofs %%%%%%%%%%%%%%%%%%%%%%%%%%%%%%
\section{Proofs}

\subsection{Proof of Lemma \ref{lem:r_ind}}

%\begin{proof}
Obviously, these vectors are all zero mean Gaussian random vectors. Independence follows from two facts:
\begin{enumerate}
	\item $\{a_i\}_{i=1}^m$ and $\{b_i\}_{i=1}^m$ are independent Gaussian vectors. 
	\item $\bbE[\ba_i \ta_i^T] = \bbE [U^T a_i a_i^T \tU] = U^T I_{n_1} \tU = 0$, and $\bbE[\bb_i \tb_i^T] = \bbE [V^T b_i b_i^T \tV] = V^T I_{n_2} \tV = 0$. (Uncorrelated Gaussian random vector are independent.)
\end{enumerate}
Covariance matrices are easy to compute. For example, $\operatorname{Cov}(\ba_i) = \bbE[\ba_i \ba_i^T] = \bbE [U^T a_i a_i^T U] = U^T I_{n_1} U = I_r$.
%\end{proof}

%%%%%%%%%%%%%%%%%%%%%%%%%%%%%% Proof r Dim Reduction %%%%%%%%%%%%%%%%%%%%%%%%%%%%%%
\subsection{Proof of Theorem \ref{thm:lin}}\label{sec:proof_lin}
We start by proving some useful lemmas.

\begin{lemma} \label{lem:independent}
$y_i$ and $\ta_i,\tb_i$ are independent.
\end{lemma}
\begin{proof}
%$(\ta_i,\tb_i) \perp (\ba_i,\bb_i)$
%$(y_i \perp (\ta_i,\tb_i))| (\ba_i,\bb_i)$
By Lemma \ref{lem:r_ind}, $\ta_i,\tb_i$ and $\ba_i,\bb_i$ are independent. By the Markov chain assumption \eqref{eq:markov}, $y_i$ and $\ta_i,\tb_i$ are conditionally independent given $\ba_i,\bb_i$. Therefore, by contraction property of conditional independence, $y_i$ and $\ta_i,\tb_i$ are independent.

When $y_i$ is a continuous random variable, the contraction property can be proved as follows:
\begin{align}
p(y_i,\ta_i,\tb_i) =&~ p(y_i,\ta_i,\tb_i | \ba_i,\bb_i)\cdot p(\ba_i,\bb_i) \nonumber\\
=&~ p(y_i| \ba_i,\bb_i) \cdot p(\ta_i,\tb_i | \ba_i,\bb_i)  \cdot p(\ba_i,\bb_i) \label{eq:con_ind}\\
=&~ p(y_i| \ba_i,\bb_i) \cdot p(\ta_i,\tb_i)  \cdot p(\ba_i,\bb_i) \label{eq:ind} \\
=&~ p(y_i) \cdot p(\ta_i,\tb_i). \nonumber
\end{align}
Equation \eqref{eq:con_ind} follows from the conditional independence of $y_i$ and $(\ta_i,\tb_i)$ given $(\ba_i,\bb_i)$. Equation \eqref{eq:ind} follows from the independence between $(\ta_i,\tb_i)$ and $(\ba_i,\bb_i)$.
\end{proof}

\begin{lemma}\label{lem:mean_var}
\[
\bbE\left[\ba_i y_i \bb_i^T \right] = Q, \quad \bbE\left[\norm{\ba_i y_i \bb_i^T-Q}_\rmF^2 \right] \leq r^2\sigma_{y|a,b}^2 + \sigma^2,
\]
\[
\bbE\left[\left|y_i\right|^2\right] \leq \sigma_{y|a,b}^2+ \tau_0^2, \quad \bbE\left[\norm{\ba_iy_i}_2^2\right] \leq r\sigma_{y|a,b}^2+\tau_1^2, \quad \bbE\left[\norm{\bb_iy_i}_2^2 \right] \leq r\sigma_{y|a,b}^2+\tau_2^2.
\]
\end{lemma}
\begin{proof}
We prove the equality using the tower property of conditional expectation:
\begin{align*}
\bbE\left[\ba_i y_i \bb_i^T \right] = &~ \bbE\left[\ba_i ~\bbE[ y_i | a_i,b_i ]~\bb_i^T \right] = \bbE\left[\ba_i f(\ba_i,\bb_i)\bb_i^T \right] = Q.
\end{align*}
For the first inequality, note that
\begin{align*}
\norm{\ba_i y_i \bb_i^T-Q}_\rmF^2 =&~ \norm{\ba_i \left[y_i-f(\ba_i,\bb_i)\right] \bb_i^T}_\rmF^2 + 2\left<\ba_i \left[y_i-f(\ba_i,\bb_i)\right] \bb_i^T,\ba_i f(\ba_i,\bb_i) \bb_i^T-Q\right>\\
&~ + \norm{\ba_i f(\ba_i,\bb_i) \bb_i^T-Q}_\rmF^2 \\
=&~ |y_i-f(\ba_i,\bb_i)|^2 \cdot\norm{\ba_i}_2^2\cdot \norm{\bb_i}_2^2 + 2 \left[y_i-f(\ba_i,\bb_i)\right] \cdot \left<\ba_i \bb_i^T,\ba_i f(\ba_i,\bb_i) \bb_i^T-Q\right> \\
&~ + \norm{\ba_i f(\ba_i,\bb_i) \bb_i^T-Q}_\rmF^2.
\end{align*}
Hence we have
\begin{align*}
\bbE\left[\norm{\ba_i y_i \bb_i^T-Q}_\rmF^2~ \middle| ~a,b \right] \leq \sigma_{y|a,b}^2\cdot\norm{\ba_i}_2^2\cdot \norm{\bb_i}_2^2 + \norm{\ba_i f(\ba_i,\bb_i) \bb_i^T-Q}_\rmF^2.
\end{align*}
Therefore,
\begin{align*}
\bbE\left[ \norm{\ba_i y_i \bb_i^T-Q}_\rmF^2 \right] = &~ \bbE\left[ \bbE\left[ \norm{\ba_i y_i \bb_i^T-Q}_\rmF^2~ \middle|~ a,b \right] \right] \\
\leq &~ \sigma_{y|a,b}^2\cdot\bbE\left[\norm{\ba_i}_2^2\right] \cdot \bbE\left[\norm{\bb_i}_2^2\right] + \bbE\left[\norm{\ba_i f(\ba_i,\bb_i) \bb_i^T-Q}_\rmF^2\right]\\
= &~ r^2\sigma_{y|a,b}^2 + \sigma^2.
\end{align*}
The other inequalities can be proved similarly.
\end{proof}

Next, we prove Theorem \ref{thm:lin}.
\begin{proof}[Proof of Theorem \ref{thm:lin}]
Since
\begin{align}
a_i y_i b_i^T = &~  (UU^T + \tU\tU^T)a_i y_i b_i^T (VV^T+\tV\tV^T)  \nonumber \\
= &~ U\ba_i y_i \bb_i^TV^T + \tU\ta_i y_i \tb_i^T\tV^T + U\ba_i y_i \tb_i^T\tV^T +\tU\ta_i y_i \bb_i^TV^T, \label{eq:split}
\end{align}
we have
\begin{align}
\bbE\left[a_i y_i b_i^T \right] = &~ \bbE\left[U\ba_i y_i \bb_i^TV^T\right] + \bbE\left[\tU\ta_i y_i \tb_i^T\tV^T\right] + \bbE\left[U\ba_i y_i \tb_i^T\tV^T\right] + \bbE\left[\tU\ta_i y_i \bb_i^TV^T\right] \nonumber\\
= &~ U\bbE\left[\ba_i y_i \bb_i^T\right]V^T + \tU\bbE\left[\ta_i\right]  \bbE\left[y_i\right] \bbE\left[\tb_i^T\right]\tV^T + U\bbE\left[\ba_i y_i\right] \bbE\left[\tb_i^T\right]\tV^T + \tU\bbE\left[\ta_i\right]\bbE\left[ y_i \bb_i^T\right]V^T \nonumber\\
=&~  UQV^T + 0 + 0 + 0 \nonumber\\
=&~ X. \label{eq:mean}
\end{align}
The second line follows from independence of $y_i,\ta_i,\tb_i$ (see Lemma \ref{lem:independent}).
Note that
\begin{align*}
\norm{a_i y_i b_i^T -X}_\rmF^2  = &~ \norm{U(\ba_i y_i \bb_i^T-Q)V^T}_\rmF^2 + \norm{\tU\ta_i y_i \tb_i^T\tV^T}_\rmF^2  + \norm{ U\ba_i y_i \tb_i^T\tV^T}_\rmF^2 + \norm{\tU\ta_i y_i \bb_i^TV^T}_\rmF^2\\
= &~ \norm{\ba_i y_i \bb_i^T-Q}_\rmF^2 + \norm{\ta_i y_i \tb_i^T}_\rmF^2 + \norm{\ba_i y_i \tb_i^T}_\rmF^2 + \norm{\ta_i y_i \bb_i^T}_\rmF^2\\
= &~ \norm{\ba_i y_i \bb_i^T-Q}_\rmF^2 + \norm{\ta_i}_2^2 |y_i|^2 \norm{\tb_i}_2^2 + \norm{\ba_i y_i}_2^2 \norm{\tb_i}_2^2 + \norm{\ta_i}_2^2 \norm{\bb_i y_i}_2^2.
\end{align*}
where the first equation follows from Pythagorean theorem, the second line follows from $\norm{U\Sigma V^T}_\rmF = \norm{\Sigma}_\rmF$ for matrices $U,V$ of orthonormal columns, and the third line follows from $\norm{ab^T}_\rmF = \norm{a}_2\norm{b}_2$.
By Lemma \ref{lem:mean_var},
\begin{align}
\bbE\left[\norm{a_i y_i b_i^T -X}_\rmF^2 \right] \leq &~ (r^2\sigma_{y|a,b}^2+\sigma^2) + (n_1-r)(n_2-r)(\sigma_{y|a,b}^2+\tau_0^2) \nonumber\\
&~ + (n_2-r)(r\sigma_{y|a,b}^2+\tau_1^2) + (n_1-r)(r\sigma_{y|a,b}^2+\tau_2^2) \nonumber\\
= &~ n_1n_2\sigma_{y|a,b}^2 + \sigma^2+(n_1-r)(n_2-r)\tau_0^2+(n_2-r)\tau_1^2+ (n_1-r)\tau_2^2. \label{eq:variance}
\end{align}

By \eqref{eq:mean} and \eqref{eq:variance}, and the independence between $\{a_i y_i b_i^T\}_{i=1}^m$, we have
\[
\bbE \left[ \hX_\mathrm{lin} \right] = \frac{1}{m} \sum_{i=1}^{m}\bbE\left[ a_i y_i b_i^T\right] = X,
\]
\begin{align*}
\bbE \left[ \norm{\hX_\mathrm{lin}-X}_\rmF^2 \right] = &~ \frac{1}{m^2} \sum_{i=1}^{m}\bbE\left[\norm{ a_i y_i b_i^T - X}_\rmF^2 \right] \\
\leq &~ \frac{n_1n_2\sigma_{y|a,b}^2 + \sigma^2+(n_1-r)(n_2-r)\tau_0^2+(n_2-r)\tau_1^2+ (n_1-r)\tau_2^2}{m}.
\end{align*}
\end{proof}

%%%%%%%%%%%%%%%%%%%%%%%%%%%%%% COR 1 %%%%%%%%%%%%%%%%%%%%%%%%%%%%%%
\subsection{Proof of Corollary \ref{cor:lin}}

\begin{lemma} \label{lem:subspace}
\begin{align*}
\max\left\{ \norm{\tU^T\hU}_\rmF,~\norm{\tV^T\hV}_\rmF \right\} \leq \frac{1}{\sigma_r}\norm{X-\hU\hSigma\hV^T}_\rmF.
\end{align*}
\end{lemma}
\begin{proof}[Proof of Lemma \ref{lem:subspace}]
We only prove the bound for $\norm{\tU^T\hU}_\rmF$. The bound for $\norm{\tV^T\hV}_\rmF$ can be proved similarly. Let $\widetilde{\hU}\in\bbR^{n_1\times (n_1-r)}$ denote a matrix of orthonormal columns that satisfies $\widetilde{\hU}^T\hU = 0$, then
\begin{align}
 \norm{\tU^T\hU}_\rmF = &~\norm{\widetilde{\hU}^T U}_\rmF \leq \frac{1}{\sigma_r} \norm{\widetilde{\hU}^T UQV^T}_\rmF = \frac{1}{\sigma_r} \norm{\widetilde{\hU}^T (UQV^T-\hU\hSigma\hV^T)}_\rmF \nonumber\\
= &~ \frac{1}{\sigma_r} \norm{\widetilde{\hU}^T (X-\hU\hSigma\hV^T)}_\rmF \leq \frac{1}{\sigma_r} \norm{\widetilde{\hU}^T}_2 \norm{X-\hU\hSigma\hV^T}_\rmF \leq \frac{1}{\sigma_r}  \norm{X-\hU\hSigma\hV^T}_\rmF. \nonumber
\end{align}
Here, the first equation is due to the following two identities:
\begin{align*}
&\norm{\tU^T\hU}_\rmF^2 = \norm{\tU\tU^T\hU}_\rmF^2 = \norm{\hU}_\rmF^2-\norm{UU^T\hU}_\rmF^2 = r- \norm{U^T\hU}_\rmF^2,\\
&\norm{\widetilde{\hU}^T U}_\rmF^2 = \norm{\widetilde{\hU}\widetilde{\hU}^TU}_\rmF^2 = \norm{U}_\rmF^2-\norm{\hU\hU^TU}_\rmF^2 = r- \norm{\hU^TU}_\rmF^2.
\end{align*}
\end{proof}

\begin{proof}[Proof of Corollary \ref{cor:lin}]
Obviously,
\begin{align}
\norm{\hU\hSigma\hV^T-X}_\rmF \leq  \norm{\hU\hSigma\hV^T-\hX_\mathrm{lin}}_\rmF+\norm{\hX_\mathrm{lin}-X}_\rmF \leq 2\norm{\hX_\mathrm{lin}-X}_\rmF, \label{eq:rankr}
\end{align}
which follows from triangle inequality, and the fact that $\hU\hSigma\hV^T$ is the best rank-$r$ approximation of $\hX_\mathrm{lin}$. Hence, by Lemma \ref{lem:subspace} and Jensen's inequality,
\[
\max\left\{\bbE\left[\norm{\tU^T\hU}_\rmF\right],~ \bbE\left[\norm{\tV^T\hV}_\rmF\right]\right\} \leq \frac{2}{\sigma_r} \bbE\left[\norm{X-\hX_\mathrm{lin}}_\rmF\right] \leq \frac{2}{\sigma_r}\sqrt{\bbE\left[\norm{X-\hX_\mathrm{lin}}_\rmF^2\right]}.
\]
Clearly, $\sigma_r,\sigma_{y|a,b},\sigma,\tau_0,\tau_1,\tau_2$ are all independent of $n_1$, $n_2$, and $m$. Since $r=O(1)$, we complete the proof by applying the mean squared error bound in Theorem \ref{thm:lin}.
\end{proof}

%%%%%%%%%%%%%%%%%%%%%%%%%%%%%% Proof s1 s2 Var Selection %%%%%%%%%%%%%%%%%%%%%%%%%%%%%%
\subsection{Proof of Theorem \ref{thm:nonlin}}\label{sec:proof_nonlin}
First, we establish some useful lemmas. Define
\[
\Delta \defeq (\Omega-\Omega) \bigcap \calB_{n_1\times n_2},
\]
\[
\norm{Y}_{\Delta^\circ} \defeq \sup\limits_{X\in\Delta} \left<Y,X\right>.
\]
Here, ${\Delta^\circ}$ is the polar set of $\Delta$. Lemma \ref{lem:seminorm} follows from the properties of polar sets.
\begin{lemma} \label{lem:seminorm}
For symmetric set $\Delta$, $\norm{\cdot}_{\Delta^\circ}$ is a pseudo-norm, or equivalently
\begin{enumerate}
	\item $\norm{Y}_{\Delta^\circ} \geq 0$, and $\norm{0}_{\Delta^\circ} = 0$.
	\item $\norm{cY}_{\Delta^\circ} = |c| \cdot\norm{Y}_{\Delta^\circ}$.
	\item $\norm{Y_1+Y_2}_{\Delta^\circ} \leq \norm{Y_1}_{\Delta^\circ} + \norm{Y_2}_{\Delta^\circ} $.
\end{enumerate}
Properties 2 and 3 imply that $\norm{\cdot}_{\Delta^\circ}$ is convex.
\end{lemma}

\begin{lemma} \label{lem:proj}
If $\Omega$ is a cone, then
\[
\norm{P_\Omega \hX_\mathrm{lin} - X}_\rmF \leq 2\norm{\hX_\mathrm{lin} - X}_{\Delta^\circ}
\]
\end{lemma}

\begin{proof}
Since $\Omega$ is a cone, we have $(\Omega-\Omega)\bigcap t\calB_{n_1\times n_2} = t\Delta$ for $t>0$. Moreover,
\begin{align*}
\frac{1}{t}\norm{Y}_{(t\Delta)^\circ} = \frac{1}{t}\sup\limits_{X\in t\Delta} \left<Y,X\right> = \sup\limits_{X\in\Delta} \left<Y,X\right> = \norm{Y}_{\Delta^\circ}.
\end{align*}
By \cite[Corollary 8.3]{Plan2014}, for every $t>0$ we have
\[
\norm{P_{\Omega} \hX_\mathrm{lin} -X }_\rmF \leq \max\left\{t, ~\frac{2}{t}\norm{\hX_\mathrm{lin} - X}_{(t\Delta)^\circ}\right\} = \max\left\{t,~2\norm{\hX_\mathrm{lin} - X}_{\Delta^\circ}\right\}.
\]
Lemma \ref{lem:proj} follows from letting $t$ go to $0$.
\end{proof}

The next lemma follows trivially from the definitions of $\Omega_1$ and $\Omega_2$.
\begin{lemma} \label{lem:hX2}
Suppose $\Omega_{12} = \Omega_1 \bigcap \Omega_2 = \{X\in\bbR^{n_1\times n_2}: \norm{X^{(:,k)}}_0\leq s_1,~\forall k\in [n_2],~\norm{X}_{0,c}\leq s_2\}$. Then
\[\hX_2 = P_{\Omega_2} P_{\Omega_1} \hX_\mathrm{lin} = P_{\Omega_{12}} \hX_\mathrm{lin}.\] 
\end{lemma}

\begin{lemma} \label{lem:max}
Suppose $\Delta_{12} = (\Omega_{12}-\Omega_{12}) \bigcap \calB_{n_1\times n_2}$. Then
\[
\norm{Y}_{\Delta_{12}^\circ} \leq \min\left\{\norm{Y}_\rmF,~\sqrt{2s_1s_2} \max_{j,k}\left|Y^{(j,k)}\right|\right\}.
\]
\end{lemma}
\begin{proof}
By Cauchy-Schwarz inequality,
\begin{align}
\norm{Y}_{\Delta_{12}^\circ} = \sup\limits_{X\in\Delta_{12}} \left<Y,X\right> \leq \sup\limits_{X\in\Delta_{12}} \norm{X}_\rmF\norm{Y}_\rmF = \norm{Y}_\rmF. \label{eq:c-s}
\end{align}
Since
\[
\Delta_{12} \subset \{X\in\bbR^{n_1\times n_2}: \norm{X}_0\leq 2s_1s_2, \norm{X}_\rmF\leq 1\} \subset \{X\in\bbR^{n_1\times n_2}: \norm{\vect(X)}_1\leq \sqrt{2s_1s_2}\} \eqdef \Delta_{\ell_1},
\]
By H\"{o}lder's inequality,
\begin{align}
\norm{Y}_{\Delta_{12}^\circ} = \sup\limits_{X\in\Delta_{12}} \left<Y,X\right>\leq \sup\limits_{X\in\Delta_{\ell_1}} \left<Y,X\right> \leq \sup\limits_{X\in\Delta_{\ell_1}} \norm{\vect(X)}_1\norm{\vect(Y)}_\infty = \sqrt{2s_1s_2}\max_{j,k} \left|Y^{(j,k)}\right|. \label{eq:holder}
\end{align}
The lemma follows from \eqref{eq:c-s} and \eqref{eq:holder}.
\end{proof}

\begin{lemma} \label{lem:degen_gauss}
Suppose $u\sim N(0,I_n)$, $\tilde{u}\sim N(0,P)$ and $P\in\bbR^{n\times n}$ is a projection matrix. Then for a convex function $g(\cdot)$, we have $\bbE[g(\tilde{u})]\leq \bbE[g(u)]$.
\end{lemma}
\begin{proof}
Let $\bar{u} \sim  N(0,I-P)$ be independent from $\tilde{u}$, then $\tilde{u}+\bar{u}$ have the same distribution as $u$.
\[
\bbE[g(\tilde{u})] = \bbE[g(\tilde{u}+ \bbE[\bar{u}])] \leq \bbE[g(\tilde{u}+ \bar{u})] = \bbE[g(u)],
\]
where the inequality follows from Jensen's inequality.
\end{proof}

\begin{lemma} \label{lem:order}
Suppose $u_{i}^{(j)}$ ($i=1,2,\cdots,m$, $j=1,2,\cdots,n$) are i.i.d. Gaussian random variables $N(0,1)$. Then
\[
\bbE\left[\max_{j\in[n]} \sqrt{\sum_{i=1}^{m}\sigma_i^2 u_i^{(j)2}}\right] \leq \sqrt{(3\log n + 2) \sum_{i=1}^m \sigma_i^2}.
\]
\end{lemma}

\begin{proof}
Let $d^{(j)} \defeq \sqrt{\sum_{i=1}^{m}\sigma_i^2 u_i^{(j)2}}$, and $d \defeq \max_{j\in[n]} d^{(j)}$.
By Jensen's inequality,
\[
e^{t\bbE[d^2]} \leq \bbE\left[e^{td^2}\right] \leq \sum_{j=1}^{n}\bbE\left[e^{td^{(j)2}}\right] = n \prod_{i=1}^m\bbE\left[e^{t\sigma_i^2 u_i^{(1)2}}\right] =  n \prod_{i=1}^m \left(1-2t\sigma_i^2\right)^{-\frac{1}{2}},\quad \forall~ 0<t<\frac{1}{2\max_i\sigma_i^2}.
\]
Therefore,
\[
\bbE[d^2] \leq \frac{\log n}{t} - \frac{1}{2t}\sum_{i=1}^{m}\log (1-2t\sigma_i^2),\quad \forall~ 0<t<\frac{1}{2\max_i\sigma_i^2}.
\]
It is easy to verify that $-\frac{1}{2}\log(1-2x) \leq 2x$ for $0<x<\frac{1}{3}$. Choose $t = \frac{1}{3\sum_{i=1}^m \sigma_i^2}$, then $0<t\sigma_i^2<\frac{1}{3}$. Hence
\[
\bbE[d^2] \leq \frac{\log n}{t} +\frac{1}{t}\sum_{i=1}^m 2t\sigma_i^2 = (3\log n + 2) \sum_{i=1}^m \sigma_i^2,
\]
\[
\bbE[d] \leq \sqrt{\bbE[d^2]} \leq \sqrt{(3\log n + 2) \sum_{i=1}^m \sigma_i^2}.
\]
\end{proof}

Next, we prove Theorem \ref{thm:nonlin}.
\begin{proof}[Proof of Theorem \ref{thm:nonlin}]
By \eqref{eq:split} and triangle inequality,
\begin{align}
\norm{\hX_\mathrm{lin} - X}_{\Delta_{12}^\circ} \leq &~ \norm{ U \left(\frac{1}{m}\sum_{i=1}^m\ba_i y_i \bb_i^T -Q\right)V^T}_{\Delta_{12}^\circ} + \norm{\tU\left(\frac{1}{m}\sum_{i=1}^m\ta_i y_i \tb_i^T\right)\tV^T}_{\Delta_{12}^\circ}   \nonumber\\
&~ + \norm{U \left(\frac{1}{m}\sum_{i=1}^m\ba_i y_i \tb_i^T \right)\tV^T}_{\Delta_{12}^\circ} + \norm{\tU\left(\frac{1}{m}\sum_{i=1}^m\ta_i y_i \bb_i^T\right)V^T }_{\Delta_{12}^\circ} \nonumber\\
\eqdef &~ T_1+T_2+T_3+T_4. \label{eq:split2}
\end{align}
Next, we bound the expectation of the four terms. For $T_1$, we use \eqref{eq:c-s}:
\begin{align}
\bbE[T_1] \leq &~ \bbE\left[\norm{U \left(\frac{1}{m}\sum_{i=1}^m\ba_i y_i \bb_i^T -Q\right)V^T}_\rmF \right] \nonumber\\
= &~ \bbE\left[\norm{\left(\frac{1}{m}\sum_{i=1}^m\ba_i y_i \bb_i^T -Q\right)}_\rmF \right]  \nonumber\\
\leq &~ \sqrt{\bbE\left[\norm{\left(\frac{1}{m}\sum_{i=1}^m\ba_i y_i \bb_i^T -Q\right)}_\rmF^2 \right]} \nonumber\\
\leq &~ \sqrt{\frac{r^2\sigma_{y|a,b}^2+\sigma^2}{m}}. \label{eq:T1}
\end{align}

Suppose $u_i\sim N(0,I_{n_1})$, $v_i\sim N(0,I_{n_2})$, $\{u_i\}_{i=1}^m$, $\{v_i\}_{i=1}^m$, and $\{y_i\}_{i=1}^m$, $\{\ba_i\}_{i=1}^m$, $\{\bb_i\}_{i=1}^m$ are independent. Replacing $\tU\ta_i,\tV\tb_i$ in $T_2$ by $u_i,v_i$, by Lemma \ref{lem:degen_gauss} and \eqref{eq:holder},
\begin{align}
\bbE[T_2] \leq \bbE\left[\norm{\frac{1}{m}\sum_{i=1}^m u_i y_i v_i^T}_{\Delta_{12}^\circ}\right]\leq \frac{\sqrt{2s_1s_2}}{m}~ \bbE\left[\max_{j,k} \Big|\sum_{i=1}^m u_i^{(j)}y_i v_i^{(k)}\Big|\right]. \label{eq:T2_first}
\end{align}
Conditioned on $\{y_i,v_i\}_{i=1}^{m}$, the distribution of $\sum_{i=1}^m u_i^{(j)}y_i v_i^{(k)}$ is $N(0,\sum_{i=1}^m y_i^2 v_i^{(k)2})$. By Lemma \ref{lem:order},
\begin{align*}
\bbE\left[\max_{j,k} \Big|\sum_{i=1}^m u_i^{(j)}y_i v_i^{(k)}\Big|~ \middle|~ \{y_i,v_i\}_{i=1}^{m}\right] \leq &~ \max_k \sqrt{(3\log n_1 +2)}\cdot\sqrt{\sum_{i=1}^m y_i^2 v_i^{(k)2}}\\
 \leq &~ 2\sqrt{\log n_1} \max_k \sqrt{\sum_{i=1}^m y_i^2 v_i^{(k)2}}.
\end{align*}
The second line follows from $n_1\geq 8$. Conditioned on $\{y_i\}_{i=1}^{m}$ alone, apply Lemma \ref{lem:order} one more time,
\begin{align*}
\bbE\left[\max_{j,k} \Big|\sum_{i=1}^m u_i^{(j)}y_i v_i^{(k)}\Big|~ \middle|~ \{y_i\}_{i=1}^{m}\right] \leq &~ 2\sqrt{\log n_1} \bbE\left[ \max_k \sqrt{\sum_{i=1}^m y_i^2 v_i^{(k)2}} ~ \middle|~ \{y_i\}_{i=1}^{m}\right] \\
\leq &~ 4\sqrt{\log n_1 \log n_2} \sqrt{\sum_{i=1}^m y_i^2}.
\end{align*}
By \eqref{eq:T2_first},
\begin{align*}
\bbE[T_2] \leq &~ \frac{\sqrt{2s_1s_2}}{m}~ \bbE\left[\max_{j,k} \Big|\sum_{i=1}^m u_i^{(j)}y_i v_i^{(k)}\Big|\right] \\
\leq &~ \frac{4\sqrt{2s_1s_2\log n_1\log n_2}}{m} \bbE\left[ \sqrt{\sum_{i=1}^m y_i^2} \right]\\
\leq &~ 4\sqrt{\frac{2s_1s_2\log n_1\log n_2 \cdot (\sigma_{y|a,b}^2+\tau_0^2)}{m}}.
\end{align*}

The bounds on the expectations of $T_3$ and $T_4$ can be derived similarly. 
\begin{align*}
\bbE[T_3] \leq &~ \bbE\left[\norm{\frac{1}{m}\sum_{i=1}^m U\ba_i y_i v_i^T}_{\Delta_{12}^\circ}\right] \leq \frac{\sqrt{2s_1s_2}}{m}~ \bbE\left[\max_{j,k} \Big|\sum_{i=1}^m (U\ba_i y_i)^{(j)}v_i^{(k)}\Big|\right] \\
\leq &~ \frac{2\sqrt{2s_1s_2\log n_2}}{m}~\bbE\left[\max_{j} \sqrt{\sum_{i=1}^m (U\ba_i y_i)^{(j)2} } \right] \\
\leq &~ \frac{2\sqrt{2s_1s_2\log n_2}}{m}~\bbE\left[\sqrt{\sum_{j=1}^{n_1}\sum_{i=1}^m (U\ba_i y_i)^{(j)2} } \right] \\
\leq &~ \frac{2\sqrt{2s_1s_2\log n_2}}{m}~\sqrt{\bbE\left[\sum_{i=1}^m \norm{\ba_i y_i}_2^2 \right]} \\
\leq &~ 2\sqrt{\frac{2s_1s_2\log n_2 \cdot (r\sigma_{y|a,b}^2+\tau_1^2)}{m}}. \\
\bbE[T_4] \leq &~ 2\sqrt{\frac{2s_1s_2\log n_1 \cdot(r\sigma_{y|a,b}^2+\tau_2^2)}{m}}.
\end{align*}

By Lemma \ref{lem:proj} and \eqref{eq:split2}, we have
\begin{align*}
\bbE\left[\norm{\hX_2 - X}_\rmF\right] \leq &~ 2 \bbE\left[\norm{\hX_\mathrm{lin} - X}_{\Delta_{12}^\circ}\right] \leq 2 \bbE[T_1] + 2 \bbE[T_2] + 2 \bbE[T_3] + 2 \bbE[T_4] \\
\leq &~ 2\sqrt{\frac{r^2\sigma_{y|a,b}^2+\sigma^2}{m}} + 8\sqrt{\frac{2s_1s_2\log n_1\log n_2 \cdot (\sigma_{y|a,b}^2+\tau_0^2)}{m}} \\
&~ + 4\sqrt{\frac{2s_1s_2\log n_2 \cdot (r\sigma_{y|a,b}^2+\tau_1^2)}{m}} + 4\sqrt{\frac{2s_1s_2\log n_1 \cdot(r\sigma_{y|a,b}^2+\tau_2^2)}{m}}.
\end{align*}
\end{proof}

%%%%%%%%%%%%%%%%%%%%%%%%%%%%%% COR 2 %%%%%%%%%%%%%%%%%%%%%%%%%%%%%%
\subsection{Proof of Corollary \ref{cor:nonlin}}
%\begin{proof}
Since $\hX_3 = P_{\Omega_3} \hX_2$, and $X\in \Omega_3$, we have
\[
\norm{\hX_3-X}_\rmF\leq \norm{\hX_3-\hX_2}_\rmF + \norm{\hX_2-X}_\rmF\leq 2\norm{\hX_2-X}_\rmF.
\]
Similarly, $\hU'\hSigma'\hV'^T = P_{\Omega_r}\hX_3$, and $X\in \Omega_r$, hence
\[
\norm{\hU'\hSigma'\hV'^T-X}_\rmF\leq \norm{\hU'\hSigma'\hV'^T-\hX_3}_\rmF + \norm{\hX_3-X}_\rmF\leq 2\norm{\hX_3-X}_\rmF \leq 4\norm{\hX_2-X}_\rmF.
\]
By Lemma \ref{lem:subspace},
\[
\max\left\{\bbE\left[\norm{\tU^T\hU'}_\rmF\right],~\bbE\left[\norm{\tV^T\hV'}_\rmF\right] \right\} \leq \frac{1}{\sigma_r}\bbE\left[\norm{X-\hU'\hSigma'\hV'^T}_\rmF\right] \leq \frac{4}{\sigma_r} \bbE\left[\norm{\hX_2 - X}_\rmF\right].
\]
Corollary \ref{cor:nonlin} follows from Theorem \ref{thm:nonlin}, and the fact that $r,\sigma_{y|a,b},\sigma,\tau_0,\tau_1,\tau_2$ are all independent of $n_1$, $n_2$, and $m$.
%\end{proof}

%%%%%%%%%%%%%%%%%%%%%%%%%%%%%% Proof for Optimal %%%%%%%%%%%%%%%%%%%%%%%%%%%%%%

\subsection{Proof of Theorem \ref{thm:optimal}}\label{sec:proof_optimal}

We need the following lemmas for the proof of Theorem \ref{thm:optimal}.
\begin{lemma} \label{lem:max2}
Suppose $\Delta_{r} = (\Omega_{r}-\Omega_{r}) \bigcap \calB_{n_1\times n_2}$. Then
\[
\norm{Y}_{\Delta_{r}^\circ} \leq \min\left\{\norm{Y}_\rmF,~\sqrt{2r} \norm{Y}\right\}.
\]
\end{lemma}
\begin{proof}
By an argument similar to that in \eqref{eq:c-s}, $\norm{Y}_{\Delta_{r}^\circ} \leq \norm{Y}_\rmF$.
Since
\[
\Delta_{r} \subset \{X\in\bbR^{n_1\times n_2}: \rank(X)\leq 2r, \norm{X}_\rmF\leq 1\} \subset \{X\in\bbR^{n_1\times n_2}: \norm{X}_*\leq \sqrt{2r}\} \eqdef \Delta_*,
\]
By H\"{o}lder's inequality,
\begin{align*}
\norm{Y}_{\Delta_{r}^\circ} = \sup\limits_{X\in\Delta_{r}} \left<Y,X\right>\leq \sup\limits_{X\in\Delta_*} \left<Y,X\right> \leq \sup\limits_{X\in\Delta_*} \norm{X}_*\norm{Y} = \sqrt{2r}\norm{Y}.
\end{align*}
\end{proof}

\begin{lemma} \label{lem:use_exp}
If $\{y_i\}_{i=1}^m$ are i.i.d. light-tailed random variables defined by \eqref{eq:light}, then
\begin{align*}
\bbP\left[\max_i |y_i| > t \log m\right] \leq C m^{1-ct}.
\end{align*}
\end{lemma}
\begin{proof}
\begin{align*}
\bbP\left[\max_i |y_i| > t \log m\right] \leq \sum_i \bbP\left[ |y_i| > t \log m\right] \leq m Ce^{-c t \log m} = C m^{1-ct}.
\end{align*}
\end{proof}

We need the following matrix Bernstein inequality.
\begin{lemma} \label{lem:bern}\cite[Theorem 6.2]{Tropp2011}
Suppose $\{X_i\}_{i=1}^m$ are $n\times n$ symmetric independent random matrices,
\[
\bbE[X_i] = 0,\quad  \bbE\left[X_i^k\right] \preceq \frac{k!}{2}\cdot R^{k-2}A_i^2,\quad \sigma_A^2 \defeq \norm{\sum_i A_i^2}. 
\] 
Then for all $t\geq 0$, we have
\[
\bbP\left[\lambda_{\max}\left(\sum_i X_i\right)\geq t \right] \leq n \cdot \exp\left(\frac{-t^2/2}{\sigma_A^2+Rt}\right).
\]
\end{lemma}

Next, we prove Theorem \ref{thm:optimal}.
\begin{proof}[Proof of Theorem \ref{thm:optimal}]
Similar to \eqref{eq:split2}, we have
\begin{align}
\norm{\hX_\mathrm{lin} - X}_{\Delta_{r}^\circ} \leq &~ \norm{ U \left(\frac{1}{m}\sum_{i=1}^m\ba_i y_i \bb_i^T -Q\right)V^T}_{\Delta_{r}^\circ} + \norm{\tU\left(\frac{1}{m}\sum_{i=1}^m\ta_i y_i \tb_i^T\right)\tV^T}_{\Delta_{r}^\circ}   \nonumber\\
&~ + \norm{U \left(\frac{1}{m}\sum_{i=1}^m\ba_i y_i \tb_i^T \right)\tV^T}_{\Delta_{r}^\circ} + \norm{\tU\left(\frac{1}{m}\sum_{i=1}^m\ta_i y_i \bb_i^T\right)V^T }_{\Delta_{r}^\circ} \nonumber\\
\eqdef &~ T_1+T_2+T_3+T_4. \label{eq:split3}
\end{align}
Similar to \eqref{eq:T1}, $\bbE[T_1] \leq \sqrt{\frac{r^2\sigma_{y|a,b}^2+\sigma^2}{m}}$.

Suppose $u_i\sim N(0,I_{n_1})$, $v_i\sim N(0,I_{n_2})$, $\{u_i\}_{i=1}^m$, $\{v_i\}_{i=1}^m$, and $\{y_i\}_{i=1}^m$, $\{\ba_i\}_{i=1}^m$, $\{\bb_i\}_{i=1}^m$ are independent. Replacing $\tU\ta_i,\tV\tb_i$ in $T_2$ by $u_i,v_i$, by Lemmas \ref{lem:degen_gauss} and \ref{lem:max2},
\begin{align}
\bbE[T_2] \leq \bbE\left[\norm{\frac{1}{m}\sum_{i=1}^m u_i y_i v_i^T}_{\Delta_{r}^\circ}\right]\leq \frac{\sqrt{2r}}{m}~ \bbE\left[ \norm{\sum_{i=1}^m u_iy_i v_i^T}\right]. \label{eq:T2_second}
\end{align}

We give the following concentration of measure bound on the spectral norm in \eqref{eq:T2_second},
\begin{align}
&~ \bbP\left[\norm{\sum_{i=1}^m u_iy_i v_i^T} \geq t^2 \sqrt{(n_1+n_2)m}\cdot\log m\right] \nonumber\\
\leq &~ \bbP\left[\norm{\sum_{i=1}^m u_iy_i v_i^T} \geq t^2 \sqrt{(n_1+n_2)m}\cdot\log m,~\max_i |y_i| \leq t \log m\right] + \bbP\left[\max_i |y_i| > t \log m\right]  \nonumber\\
\leq &~ (n_1+n_2)\cdot \exp\left(\frac{-t^4}{2t^2+6t^3}\right)+Cm^{1-ct}. \label{eq:use_conc}
\end{align}
The bounds on the first and second terms follow from Lemmas \ref{lem:bern} and \ref{lem:use_exp}, respectively. The derivation for the first bound can be found in Appendix \ref{app:spectral}.
By \eqref{eq:use_conc},
\begin{align*}
\bbP\left[\norm{\sum_{i=1}^m u_iy_i v_i^T} \geq t^2 \sqrt{(n_1+n_2)m}\cdot\log m\right]
\leq \begin{cases}
1, & \quad \text{if}~t\leq 8\log(n_1+n_2),\\
(n_1+n_2)\cdot \exp\left(\frac{-t}{8}\right)+2Cm^{1-ct}, & \quad \text{if}~t> 8\log(n_1+n_2).
\end{cases} 
\end{align*}
Hence
\begin{align*}
\bbE\left[ \norm{\sum_{i=1}^m u_iy_i v_i^T}\right] = &~ \int_{0}^{\infty} \bbP\left[\norm{\sum_{i=1}^m u_iy_i v_i^T} \geq x\right] \rmd x\\
\leq &~ \sqrt{(n_1+n_2)m}\cdot \log m\cdot \Big(64\log^2(n_1+n_2)+128\log(n_1+n_2)+128\\
&~ +\frac{16C \log(n_1+n_2)}{c\log m\cdot m^{8c\log(n_1+n_2)-1}}+\frac{2C}{c^2 \log^2m \cdot m^{8c\log(n_1+n_2)-1}} \Big)\\
\leq &~ 256(C+2) \sqrt{(n_1+n_2)m}\cdot \log m \cdot\log^2(n_1+n_2).
\end{align*}
The derivation is tedious but elementary, in which the assumptions $c>\frac{1}{8\log(n_1+n_2)}$ and $m>n_1+n_2$ are invoked.
By \eqref{eq:T2_second},
\begin{align*}
\bbE[T_2] \leq  \frac{\sqrt{2r}}{m}~ \bbE\left[ \norm{\sum_{i=1}^m u_iy_i v_i^T} \right] \leq 256\sqrt{2}(C+2) \sqrt{\frac{(n_1+n_2)r\log^2m\log^4(n_1+n_2)}{m}}
\end{align*}

It is easy to obtain bounds on the expectations of $T_3$ and $T_4$.
\begin{align}
\bbE[T_3] \leq &~ \bbE\left[\norm{\frac{1}{m}\sum_{i=1}^m U\ba_i y_i v_i^T}_{\Delta_{r}^\circ}\right] \leq \frac{1}{m}~ \bbE\left[\norm{\sum_{i=1}^m (\ba_i y_i)v_i^T}_\rmF\right] \label{eq:T3}\\
\leq &~ \frac{1}{m}~ \sqrt{\bbE\left[\norm{\sum_{i=1}^m (\ba_i y_i)v_i^T}_\rmF^2\right]}\nonumber\\
\leq &~ \frac{1}{m}~ \sqrt{m \bbE\left[\norm{\ba_i y_i}_2^2\right] \bbE\left[\norm{v_i}_2^2\right]} \nonumber\\
\leq &~ \sqrt{\frac{n_2 (r\sigma_{y|a,b}^2+\tau_1^2)}{m}}, \nonumber\\
\bbE[T_4] \leq &~ \sqrt{\frac{n_1 (r\sigma_{y|a,b}^2+\tau_2^2)}{m}}. \nonumber
\end{align}

By Lemma \ref{lem:proj} and \eqref{eq:split3}, we have
\begin{align*}
\bbE\left[\norm{\hU\hSigma\hV^T - X}_\rmF\right] \leq &~ 2 \bbE\left[\norm{\hX_\mathrm{lin} - X}_{\Delta_{r}^\circ}\right] \leq 2 \bbE[T_1] + 2 \bbE[T_2] + 2 \bbE[T_3] + 2 \bbE[T_4] \\
\leq &~ 2\sqrt{\frac{r^2\sigma_{y|a,b}^2+\sigma^2}{m}} + 512\sqrt{2}(C+2) \sqrt{\frac{(n_1+n_2)r\log^2m\log^4(n_1+n_2)}{m}} \\
&~ + 2\sqrt{\frac{n_2 (r\sigma_{y|a,b}^2+\tau_1^2)}{m}} + 2\sqrt{\frac{n_1 (r\sigma_{y|a,b}^2+\tau_2^2)}{m}}.
\end{align*}
\end{proof}

%%%%%%%%%%%%%%%%%%%%%%%%%%%%%% Spectral norm bound %%%%%%%%%%%%%%%%%%%%%%%%%%%%%%

\subsection{Spectral Norm Bound}\label{app:spectral}
In this section, we prove the first bound in \eqref{eq:use_conc}. We have
\begin{align}
&~ \bbP\left[\norm{\sum_{i=1}^m u_iy_i v_i^T} \geq t^2 \sqrt{(n_1+n_2)m}\cdot\log m,~\max_i |y_i| \leq t \log m~\middle |~\{y_i\}_{i=1}^m \right] \nonumber\\
= &~ \bbP\left[\norm{\sum_{i=1}^m u_iy_i v_i^T} \geq t^2 \sqrt{(n_1+n_2)m}\cdot\log m~\middle |~\{y_i\}_{i=1}^m \right] \cdot \ind{\max_i |y_i| \leq t \log m} \nonumber\\
\leq &~ (n_1+n_2)\cdot \exp\left(\frac{-t^4 (n_1+n_2)m\log^2 m /2}{t^2(n_1+n_2)m\log^2 m+et^3 (n_1+n_2)m\log^2 m}\right) \label{eq:use_bern}\\
\leq &~ (n_1+n_2)\cdot \exp\left(\frac{-t^4}{2t^2+6t^3}\right), \nonumber
\end{align}
Next, we show how \eqref{eq:use_bern} follows from the matrix Bernstein inequality in Lemma \ref{lem:bern}. 
The rest of the derivation is conditioned on $\{y_i\}_{i=1}^m$ that satisfy $\max_i |y_i| \leq t \log m$, hence $\sum_i y_i^2 \leq t^2 m (\log m)^2$. Define $(n_1+n_2)\times (n_1+n_2)$ matrices ($i=1,2,\cdots,m$):
\[
X_i = \begin{bmatrix}
0 & u_iy_i v_i^T\\
v_iy_i u_i^T & 0
\end{bmatrix}.
\]
They satisfy
\[
\lambda_{\max}\left(\sum_i X_i\right) = \norm{\sum_i u_i y_i v_i^T},
\]
\[
\bbE[X_i] = 0,\quad \bbE\left[X_i^k\right]  = 0, \quad \text{if $k$ is odd,}
\]
\begin{align*}
\bbE\left[X_i^k\right] 
= &~ y_i^k (n_1+2)\cdots(n_1+k-2)(n_2+2)\cdots(n_2+k-2)\begin{bmatrix}
n_2I_{n_1} & 0\\
0 & n_1I_{n_2}
\end{bmatrix}\\
\preceq &~ \frac{k!}{2}\left[e(n_1+n_2)\max_i|y_i|\right]^{k-2} \begin{bmatrix}
y_i^2 n_2I_{n_1} & 0\\
0 & y_i^2 n_1I_{n_2}
\end{bmatrix},  \quad \text{if $k$ is even.}
\end{align*}

Let $R= e(n_1+n_2)\max_i|y_i| \leq et (n_1+n_2)\log m\leq et \sqrt{(n_1+n_2)m}\cdot \log m$,
\[
A_i^2 = \begin{bmatrix}
y_i^2 n_2I_{n_1} & 0\\
0 & y_i^2 n_1I_{n_2}
\end{bmatrix},
\]
and $\sigma_A^2 = \sum_i y_i^2 \max\{n_1,n_2\}\leq t^2(n_1+n_2)m(\log m)^2$. Then \eqref{eq:use_bern} follows from Lemma \ref{lem:bern}.

%%%%%%%%%%%%%%%%%%%%%%%%%%%%%% COR 3 Light-tailed %%%%%%%%%%%%%%%%%%%%%%%%%%%%%%
\subsection{Mildness of the Light-tailed Measurement Condition}

In this section, we demonstrate that this condition holds under reasonably mild assumptions on $f(\cdot,\cdot)$ and $y-\mu$. To this end, we review a known fact: a probability distribution is light-tailed if its moment generating function is finite at some point. This is made more precise in Proposition \ref{pro:light}, which follows trivially from Chernoff bound.
\begin{proposition}\label{pro:light}
Let $M_y(t) = \bbE\left[e^{ty}\right]$ denote the moment generating function of a random variable $y$. Then $y$ is a light-tailed random variable, if
\begin{itemize}
\item there exist $t_1>0$ and $t_2<0$ such that $M_y(t_1)<\infty$ and $M_y(t_2)<\infty$.
\item $y\geq 0$ almost surely, and there exists $t_1>0$ such that $M_y(t_1)<\infty$.
\item $y\leq 0$ almost surely, and there exists $t_2<0$ such that $M_y(t_2)<\infty$.
\end{itemize}
\end{proposition}
In the context of this paper, we have the following corollary:
\begin{corollary}\label{cor:light}
Suppose $f(\ba,\bb)$ satisfies $|f(\ba,\bb)|\leq \max\left\{C_1,~C_2\left(\norm{\ba}_2^2+\norm{\bb}_2^2\right)\right\}$ for some $C_1,C_2>0$, and $y-\mu = y-f(\ba,\bb)$ is a light-tailed random variable. Then $y$ is a light-tailed random variable.
\end{corollary}

\begin{proof}
Since $\bbP\left[\left|y\right|\geq t\right] \leq \bbP\left[\left|\mu\right|\geq t/2\right]+\bbP\left[\left|y-\mu\right|\geq t/2\right]$, and $y-\mu$ is light-tailed, it is sufficient to show that $\mu$ is light-tailed. The moment generating function of $\mu$ is
\begin{align*}
M_\mu(t) = &~ \bbE[e^{tf(\ba,\bb)}] \leq  \bbE[e^{|t|\cdot|f(\ba,\bb)|}] \leq e^{C_1|t|}\bbE\left[e^{C_2|t|\left(\norm{\ba}_2^2+\norm{\bb}_2^2\right)}\right] \\
= &~ \frac{e^{C_1|t|}}{(2\pi)^r} \int\limits_{\bb}\int\limits_{\ba} e^{\left(C_2|t|-\frac{1}{2}\right)\left(\norm{\ba}_2^2+\norm{\bb}_2^2\right)}~ \rmd \ba ~\rmd \bb,
\end{align*}
which is finite for $|t|<\frac{1}{2C_2}$. By Proposition \ref{pro:light}, $\mu$ is light-tailed. Thus the proof is complete.
\end{proof}

\end{document}